\documentclass[10pt]{article} 
\usepackage{graphicx}
\usepackage{graphics}
\usepackage{diagbox}
\usepackage{color}
\usepackage{ulem}

\usepackage{amsfonts,amsmath,amsthm,amssymb}
\begin{document}

\newtheorem{thm}{Theorem}[section]
\newtheorem{cor}[thm]{Corollary}
\newtheorem{lem}[thm]{Lemma}{\rm}
\newtheorem{prop}[thm]{Proposition}
\newtheorem{rem}[thm]{Remark}

\newtheorem{defn}[thm]{Definition}{\rm}
\newtheorem{assumption}[thm]{Assumption}
\newtheorem{ex}{Example}
\numberwithin{equation}{section}
\def\la{\langle}
\def\ra{\rangle}
\def\glexe{\leq_{gl}\,}
\def\glex{<_{gl}\,}
\def\e{{\rm e}}

\def\x{\mathbf{x}}
\def\P{\mathbf{P}}
\def\D{\mathbf{D}}
\def\h{\mathbf{h}}
\def\by{\mathbf{y}}
\def\bz{\mathbf{z}}
\def\F{\mathcal{F}}
\def\R{\mathbb{R}}
\def\T{\mathbf{T}}
\def\N{\mathbb{N}}
\def\D{\mathbf{D}}
\def\V{\mathbf{V}}
\def\U{\mathbf{U}}
\def\K{\mathbf{K}}
\def\Q{\mathbf{Q}}
\def\H{\mathbf{H}}
\def\M{\mathbf{M}}
\def\oM{\overline{\mathbf{M}}}
\def\O{\mathbf{O}}
\def\C{\mathbb{C}}
\def\P{\mathbf{P}}
\def\Z{\mathbb{Z}}
\def\H{\mathcal{H}}
\def\A{\mathbf{A}}
\def\V{\mathbf{V}}
\def\AA{\overline{\mathbf{A}}}
\def\BB{\mathbf{B}}
\def\c{\mathbf{C}}
\def\L{\mathbf{L}}
\def\bS{\mathbf{S}}
\def\H{\mathcal{H}}
\def\I{\mathbf{I}}
\def\Y{\mathbf{Y}}
\def\X{\mathbf{X}}
\def\G{\mathbf{G}}
\def\B{\mathbf{B}}
\def\f{\mathbf{f}}
\def\z{\mathbf{z}}
\def\v{\mathbf{v}}
\def\y{\mathbf{y}}
\def\d{\hat{d}}
\def\bx{\mathbf{x}}
\def\bI{\mathbf{I}}
\def\y{\mathbf{y}}
\def\g{\mathbf{g}}
\def\w{\mathbf{w}}
\def\b{\mathbf{b}}
\def\bxi{{\boldmath\xi}}
\def\a{\mathbf{a}}
\def\u{\mathbf{u}}
\def\q{\mathbf{q}}
\def\p{\mathbf{p}}
\def\e{\mathbf{e}}
\def\s{\mathcal{S}}
\def\cc{\mathcal{C}}
\def\co{{\rm co}\,}
\def\tg{\tilde{g}}
\def\tx{\tilde{\x}}
\def\tg{\tilde{g}}
\def\tA{\tilde{\A}}
\def\balpha{\mathbf{\alpha}}

\def\supmu{{\rm supp}\,\mu}
\def\supp{{\rm supp}\,}
\def\cd{\mathcal{C}_d}
\def\cok{\mathcal{C}_{\K}}
\def\cop{COP}
\def\vol{{\rm vol}\,}\newcommand{\PP}{\mathbb{P}}
\newcommand{\QQ}{\mathbb{Q}}
\newcommand{\EE}{\mathbb{E}}
\newcommand{\RR}{\mathbb{R}}
\newcommand{\NN}{\mathbb{N}}
\newcommand{\II}{\mathbb{I}}
\newcommand{\lgl}{<_{gl}}
\newcommand{\leqgl}{\leq_{gl}}
\newcommand{\jb}{\color{red}}
\newcommand{\ed}{\color{blue}}

\def\om{\mathbf{\Omega}}
\def\Hom{\mathrm{Hom}[\x]}
\def\gl{{\tiny{gl}} }
\def\glex{<_\gl}

\title{The empirical Christoffel function with applications in data analysis}
\author{Jean B. Lasserre\thanks{LAAS-CNRS and Institute of Mathematics, University of Toulouse, LAAS, 7 avenue du Colonel Roche 31077 Toulouse, France.}, Edouard Pauwels\thanks{IRIT, Universit\'e Toulouse 3 Paul Sabatier, 118 route de Narbonne, 31062 Toulouse, France.}.
}

\date{}
\maketitle
\begin{abstract}
	We illustrate the potential applications in machine learning of the Christoffel function, or more precisely,
	its empirical counterpart  associated with a counting measure uniformly supported on a finite set of points. Firstly, we provide a thresholding scheme which allows to approximate the support of a measure from a finite subset of its moments with strong asymptotic guaranties. Secondly, we provide a consistency result which relates the empirical Christoffel function and its population counterpart in the limit of large samples. Finally, we illustrate the relevance of our results on simulated and real world datasets for several applications in statistics and machine learning: (a) density and support estimation from finite samples, (b) outlier and novelty detection and (c) affine matching.
\end{abstract}

\section{Introduction}
The main claim of this paper is that the Christoffel function (a tool from Approximation Theory) can prove to be very useful machine learning applications. The Christoffel function is associated with a finite measure and a degree parameter $d$. It has an important history of research with strong connection to orthogonal polynomials \cite{szego1974orthogonal,dunkl2001orthogonal}, interpolation and approximation theory \cite{nevai1986geza,demarchi2014multivariate}. Its typical asymptotic behavior as $d$ increases is of particular interest because, in some specific settings, it provides very relevant information on the support and the density of the associated input measure. 
Important references include \cite{mate1980bernstein,mate1991szego,totik2000asymptotics,GPSS} in a single dimension, \cite{bos1994asymptotics,bos1998asymptotics,xu1999asymptoticsSimplex,berman2009bergman,kroo2012christoffelBall} for specific multivariate settings and \cite{kroo2012christoffel} for ratii of mutually absolutely continuous measures. The topic is still a subject of active research, but regarding properties of the Christoffel function, a lot of information is already available. 

The present work shows how properties of the Christoffel function can be used successfully in some machine learning applications. To the best of our our knowledge, this is the first attempt in such a context with the recent work of \cite{lasserre2016sorting} and \cite{malyshkin2015multiple}. More precisely, we consider the {\it empirical} Christoffel function, a specific case where the input measure is a scaled counting measure uniformly supported on a set (a cloud) of datapoints. This methodology has three distinguishing features:
(i) It is extremely simple and involves no optimization procedure, (ii)  it scales linearly with the number of observations (one pass over the data is sufficient), and (iii) it is affine invariant. These three features prove to be especially important in all the applications that we consider.
	
In \cite{lasserre2016sorting} we  have exhibited a striking property of some distinguished family of sum-of-squares (SOS) polynomials $(Q_d)_{d\in\N}$, indexed by their degree ($2d \in \N$), and easily computed from empirical moments associated with a cloud of $n$ points in $\R^p$ which we call $\X$. The associated family of sublevel sets $S_{\alpha,d} = \{\x:Q_d(\x)\leq \alpha\}$, for various values of $\alpha > 0$, approximates the global shape of original cloud of points $\X$. The degree index $d$ can be used as a tuning parameter, trading off regularity of the polynomial $Q_d$ with the fitness of the approximation of the shape of $\X$ (as long as the cloud contains sufficiently many points). Remarkably, even with relatively low degree $d$, the sets $S_{\alpha,d}$ capture acurately the shape of $\X$, and so provides a compact (algebraic) encoding of the cloud.

In fact the reciprocal function $\x \mapsto Q_d(\x)^{-1}$ is precisely the {\it Christoffel function} $\Lambda_{\mu_n,d}$ associated to the empirical counting measure supported on $\X$ and the degree index $d$. Some properties of the Christoffel function stemming from approximation theory suggest that it could be exploited in a statistical learning context by considering its empirical counterpart. The purpose of this work is to push this idea further. In particular we investigate (a) further properties of the Christoffel function which prove to be relevant in some  machine learning applications, (b) statistical properties of the empirical Christoffel function as well as (c) further applications to well known machine learning tasks.

\subsection*{Contributions} This paper significantly extends \cite{lasserre2016sorting} in several directions. Indeed our contribution is threefold:

I. We first provide a thresholding scheme which allows to approximate the compact support $S$ of a measure with strong asymptotic guarantees. This result rigorously establishes the property that, as $d$ increases, the scaled Christoffel function decreases to zero outside $S$ and remains positive in the interior of $S$.

II. In view of potential applications in machine learning we provide a rationale for using the empirical Christoffel function in place of its population counterpart in the limit of large sample size. We consider a compactly supported population measure $\mu$ as well as an empirical measure $\mu_n$ uniformly supported on a sample of $n$ vectors in $\RR^p$, drawn independently from $\mu$. For each fixed $d$ we show a highly desirable strong asymptotic property as $n$ increases. Namely, the empirical Christofell function  $\Lambda_{\mu_n,d}(\cdot)$ converges, {\it uniformly in $\x\in\R^p$}, to $\Lambda_{\mu,d}(\cdot)$, almost-surely with respect to the draw of the random sample.

III. We illustrate the benefits of the empirical Christoffel function in some important applications, mainly in machine learning. The rationale for such benefits builds on approximation properties of the Christoffel function combined with our consistency result. In particular, we first show on simulated data that the Christoffel function can be useful for {\it density estimation} and {\it support inference}. In \cite{lasserre2016sorting} we have described how the Christoffel function yields a simple procedure for intrusion detection in networks, and here we extend these results by performing a numerical comparison with well established methods for novelty detection on a real world dataset. Finally we show that the Christoffel function is also very useful to perform affine matching and inverse affine shuffling of a dataset.

\subsection*{Comparison with existing literature on set estimation}
Support estimation and more generaly set estimation has a long history in statistics and we intend to give a nonexhaustive overview in this section. The main question of interest is that of inferering a set (support, level sets of the density function \ldots) based on independants samples from an unknown distribution. Pioneering works include \cite{renyi1963konvexe,geffroy1964probleme} followed by \cite{chevalier1976estimation,devroye1980detection} and resulted in the introduction and first analyses for estimators based on convex hull for convex domains or union of balls for nonconvex sets. This motivated the development of minimax statistical analysis for the set estimation problem \cite{hardle1995estimation,mammen1995asymptotical,tsybakov1997nonparametric} and the introduction of more sophisticated optimal estimators, such as the excess mass estimator \cite{polonik1995measuring}. Strong relations between set estimation and density estimation lead to the development of the plugin approach for support and density level set estimation \cite{cuevas1997plugin,molchanov1998limit} with futher generalization proposed in \cite{cuevas2006plugin} and a precise minimax analysis described in \cite{rigollet2009optimal}.

These works provide a rich statitical analysis of the main estimation approaches currently available. The topic is still active with more precise questions ranging from inference of topological properties \cite{aaron2016local}, new geometric conditions \cite{cholaquidis2014poincare}, adaptivity to local properties of the underlying density \cite{patschkowski2016adaptation,singh2009adaptive}.

One of the goals of our work is the introduction of the Christoffel function as a tool to solve similar problems. This approach has several advantages
\begin{itemize}
				\item The Christoffel function allows to encode the global shape of a cloud of points in any finite dimension using a polynomial level set. This kind of encoding is relatively simple and compact. This has clear advantages, for example, the evaluation of a polynomial has a complexity which does not depend on the size of the sample used to qualibrate its coefficients and the boundary of the corresponding sublevel set as a very compact representation as an algebraic set. Furthermore, it turns out that the estimation of the empirical Christoffel function has a computational cost which is linear in the sample size. This is in contrast with distance based approaches for which membership evaluation requires to query all the sample points. As pointed out in \cite{baillo2000set}, the practical use of multidimensional set estimation techniques involves formidable computational difficulties so that simplicity arises as a major advantage in this context.
				\item The proposed approach is specific in the sense that it relies on tools which were not considered before for support estimation such as orthogonal polynomials. Topological properties of the support of the distribution or its boundary arise as major questions beyond minimax analysis \cite{aaron2016local}. In this realm, the objects which we manipulate have a simple algebraic description and could be coupled with computational real algebraic geometry tools to infer topological properties such as, for example, Betty numbers \cite{basu2005computing}. This strong algebraic structure could in principle allow to push further the statistical settings which could be handled, with, for example, notions such as singular measures and intrinsic dimension.
\end{itemize}
We see these facts as potential advantages of the Christoffel function in the context of support estimation and relevant motivation to further study the potential of this procedure in modern data analysis contexts. However, we emphasize that this work constitutes only a first step in this direction. Indeed, we are not able to provide a complete statistical efficiency analysis as precisely described in the support and set estimation literature (\textit{e.g.} \cite{cuevas2006plugin}). This would require further studies of precise properties of the Christoffel function itself which are not available given the state of knowledge for this object. We aim at providing a rationale for the proposed approach and motivation for future studies, among which a complete statistical analysis is a longer term goal.\footnote{In particular qualibration of the underlying polynomial degree as a function of the sample size is out of the scope of this paper and left for future research.}

\subsection*{Organisation of the paper}
Section \ref{sec:notations} describes the notation and definitions which will be used throughout the paper. In Section \ref{sec:christoffelFunction} we introduce the Christoffel function, 
outline some of its known properties and describe our main theoretical results. Applications are presented in Section \ref{sec:applications} where we consider both simulated and real world data as well as a comparison with well established machine learning methods. For clarity of exposition most proofs and technical details are postponed to the Appendix in Section \ref{appendix}.

\section{Notation, definitions and Preliminary results}
\label{sec:notations}
\subsection{Notation and definitions}

We fix the ambient dimension to be $p$ throughout the text. For example, we will manipulate vectors in $\RR^p$ as well as $p$-variate polynomials with real coefficients. We denote by $X$ a set of $p$ variables $X_1, \ldots, X_p$ which we will use in mathematical expressions defining polynomials. We identify monomials from the canonical basis of $p$-variate polynomials with their exponents in $\NN^p$: we associate to $\alpha = (\alpha_i)_{i =1 \ldots p} 
\in \NN^p$ the monomial $X^\alpha := X_1^{\alpha_1} X_2^{\alpha_2} \ldots X_p^{\alpha_p}$ which degree 
is $\deg(\alpha) := \sum_{i=1}^p \alpha_i=\vert\alpha\vert$. We use the expressions $\lgl$ and $\leqgl$ to denote the graded lexicographic order, a well ordering over $p$-variate monomials. This amounts to, first, use the canonical order on the degree and, second, break ties in monomials with the same degree using the lexicographic order with $X_1 =a, X_2 = b \ldots$ For example, the monomials in two variables $X_1, X_2$, of degree less or equal to $3$ listed in this order are given by: $1,\,X_1, \,X_2, \,X_1^2, \,X_1X_2, \,X_2^2,\, X_1^3,\, X_1^2X_2,\, X_1X_2^2,\, X_2^3$. 

We denote by $\NN^p_d$, the set $\left\{ \alpha \in \NN^p;\; \deg(\alpha) \leq d \right\}$ ordered by $\leqgl$. $\RR[X]$ denotes the set of $p$-variate polynomials: linear combinations of monomials with real coefficients. The degree of a polynomial is the highest of the degrees of its monomials with nonzero coefficients\footnote{For the null polynomial, we use the convention that its degree is $0$ and it is $\leqgl$ smaller than all other monomials.}. We use the same notation, $\deg(\cdot)$, to denote the degree of a polynomial or of an element of $\NN^p$. For $d \in \NN$, $\RR[X]_d$ denotes the set of $p$-variate polynomials of degree at most $d$. We set $s(d) = {p+d \choose d}$, the number of monomials of degree less or equal to $d$. 

We will denote by $\v_d(X)$ the vector of monomials of degree less or equal to $d$ sorted by $\leqgl$, i.e., $\v_d(X) := \left( X^\alpha \right)_{\alpha \in \NN^p_d} \in \RR[X]^{s(d)}_d$. With this notation, we can write a polynomial $P\in\RR[X]_d$ as $P(X) = \left\langle\p, \v_d(X)\right\rangle$ for some real vector of coefficients $\p = \left( p_{\alpha} \right)_{\alpha \in \NN_d^p} \in \RR^{s(d)}$ ordered using $\leqgl$. 
Given $\x = (x_i)_{i = 1 \ldots p} \in \RR^p$, $P(\x)$ denotes the evaluation of $P$ with the assignments $X_1 = x_1, X_2 = x_2, \ldots X_p = x_{p}$. Given a Borel probability measure $\mu$ and $\alpha \in \NN^p$, $y_{\alpha}(\mu)$ denotes the moment $\alpha$ of $\mu$, i.e., $y_{\alpha}(\mu) = \int_{\RR^p} \x^{\alpha} d\mu(\x)$. Finally for $\delta>0$ and every $\x\in\R^p$, let $\BB_{\delta}(\x):=\{\x: \Vert\x\Vert \leq \delta\}$ be the closed Euclidean ball of radius $\delta$ and centered at $\x$. We use the shorthand notation $\B$ to denote the closed Euclidean unit ball. For a given subset of Euclidean space, $A$, $\partial A$ denotes the topological boundary of $A$. 
Recall that its Lebesgue volume ${\rm vol}(\B_{\delta}(\x))$ satisfies:
\begin{align*}
	{\rm vol}(\B_{\delta}(\x)) = \frac{\pi^{\frac{p}{2}}}{\Gamma\left( \frac{p}{2} + 1 \right)}\delta^p,\qquad\forall \x\in\RR^p.
\end{align*}

Furthermore, let $\omega_p:= \frac{2\pi^{\frac{p+1}{2}}}{\Gamma\left( \frac{p+1}{2} \right)}$ denote the surface of the $p$ dimensional unit sphere in $\RR^{p+1}$. Throughout the paper, we will only consider measures of which all moments are finite.

\subsubsection*{Moment matrix}
For a finite Borel measure $\mu$ on $\R^p$ denote by ${\rm supp}(\mu)$ its support, i.e., the smallest closed
set $\om\subset\R^p$ such that $\mu(\R^p\setminus\om)=0$. The moment matrix of $\mu$, $\M_d(\mu)$, is a matrix indexed by monomials of degree at most $d$ ordered by $\leqgl$. For $\alpha,\beta \in \NN^p_d$, the corresponding entry in $\M_d(\mu)$ is defined by $\M_d(\mu)_{\alpha,\beta} := y_{\alpha +\beta}(\mu)$, the moment 
$\int\x^{\alpha+\beta}d\mu$ of $\mu$. When $p = 2$ and $d = 2$, letting $y_{\alpha} = y_{\alpha} (\mu)$ for $\alpha \in \NN_4^2$, we have
$$\M_2(\mu): \quad\begin{array}{rccccccccc}
& & &1& X_1  & X_2  & X_1^2 & X_1 X_2 &  X_2^2 \\
 & & & \\
 1  &   \quad & &1 & y_{10} & y_{01} & y_{20}& y_{11}& y_{02} \\
X_1&   \quad &&y_{10} & y_{20} & y_{11}&y_{30} & y_{21}& y_{12} \\
X_2&   \quad & &y_{01} & y_{11} & y_{02}& y_{21} &y_{12}& y_{03} \\
X_1^2& \quad & &y_{20} & y_{30} & y_{21}& y_{40} & y_{31}& y_{22} \\
X_1X_2& \quad & &y_{11} & y_{21} & y_{12}& y_{31} & y_{22}& y_{13}\\
X_2^2& \quad & &y_{02} & y_{12} &y_{03}& y_{22} & y_{13}& y_{04}\\
\end{array}.$$
The matrix $\M_d(\mu)$ is positive semidefinite for all $d \in \NN$. Indeed, for any $\p \in \RR^{s(d)}$, let $P \in \RR[X]_d$ be the polynomial with vector of coefficients $\p$; then $\p^T\M_d(\mu)\p = \int_{\RR^p} P(\x)^2 d\mu(\x) \geq 0$. We also have the identity $\M_d(\mu) = \int_{\RR^p} \v_d(\x) \v_d(\x)^T d\mu(\x)$ where the integral is understood elementwise.

\subsubsection*{Sum of squares (SOS)}

We denote by $\Sigma[X] \subset \RR[X]$ (resp. $\Sigma[X]_d \subset \RR[X]_d$), the set of polynomials (resp. polynomials of degree at most $d$) which can be written as a sum of squares of polynomials. Let $P \in \RR[X]_{2m}$ for some $m \in \NN$, then $P$ belongs to $\Sigma[X]_{2m}$ if there exists a finite $J \subset \NN$ and a family of polynomials $P_j \in \RR[X]_m$, $j \in J$, such that $P = \sum_{j \in J} P_j^2$. It is obvious that sum of squares polynomials are always nonnegative. A further interesting property is that this class of polynomials is connected with positive semidefiniteness. Indeed, $P$ belongs to $\Sigma[X]_{2m}$ if and only if
\begin{align}
	\exists \,Q \in \RR^{s(m) \times s(m)},\, Q=Q^T,\,Q\succeq 0,\, P(\x) = \v_m(\x)^T Q \v_m(\x),\, \forall \x \in \RR^p.
	\label{eq:equivSOS}
\end{align}
As a consequence, every real symmetric positive semidefinite matrix $Q \in \RR^{s(m) \times s(m)}$ defines a polynomial in $\Sigma[X]_{2m}$ by using the representation (\ref{eq:equivSOS}).

\subsubsection*{Orthonormal polynomials}
We define a classical \cite{szego1974orthogonal,dunkl2001orthogonal} family of orthonormal polynomials, $\left\{ P_\alpha \right\}_{\alpha \in \NN^p_d}$ ordered according to $\leqgl$, which satisfies for all $\alpha \in \NN^p_d$
\begin{equation}
\langle P_\alpha, P_\beta \rangle_\mu = \delta_{\alpha=\beta},\quad \langle P_\alpha, X^\beta \rangle_\mu = 0,\ {\rm if} \ \beta \lgl \alpha,
\,\langle P_\alpha, X^\alpha \rangle_\mu > 0.
\label{eq:orthoDef}
\end{equation}
Existence and uniqueness of such a family is guaranteed by the Gram-Schmidt orthonormalization process following the $\leqgl$ ordering on monomials and by the positivity of the moment matrix, see for instance \cite{dunkl2001orthogonal} Theorem 3.1.11.

Let $\D_d(\mu)$ be the lower triangular matrix of which rows are the coefficients of the polynomials $P_\alpha$ defined in (\ref{eq:orthoDef}) ordered by $\leqgl$. It can be shown that $\D_d(\mu) = \L_d(\mu)^{-T}$, where $\L_d(\mu)$ is the Cholesky factorization of $\M_d(\mu)$. Furthermore, there is a direct relation with the inverse moment matrix as $\M_d(\mu)^{-1} = \D_d(\mu)^T \D_d(\mu)$ (\cite{helton2008measure} Proof of Theorem 3.1).

\section{The Christoffel function and its empirical counterpart}
\label{sec:christoffelFunction}

\subsection{The Christoffel function}
\label{sec:christo}
Let $\mu$ be a finite Borel measure on $\R^p$ with all moments finite and such that its moment matrix $\M_d(\mu)$
is positive definite for every  $d=0,1,\ldots$. 
For every $d$, define the function $\kappa_{\mu,d}:\R^p\times\R^p\to\R$ by:
\begin{align}
	\label{eq:defKappa}
	(\x,\y)\,\mapsto\,\kappa_{\mu,d}(\x,\y)\,:=\,\sum_{\alpha\in\N^p_d}P_\alpha(\x)\,P_\alpha(\y) = \v_d(\x)^T \M_d(\mu)^{-1}\v_d(\y),
\end{align}
where the family of polynomials $\left\{P_{\alpha}  \right\}_{\alpha \in \NN_d^p}$ is defined in (\ref{eq:orthoDef}) and the last equality follows from properties of this family (see also \cite{lasserre2016sorting}). The {\it kernel} $(\x,\y)\mapsto K(\x,\y):=\,\sum_{\alpha\in\N^p}P_\alpha(\x)\,P_\alpha(\y)$
is a {\it reproducing kernel} on $L_2(\mu)$ because
\[P_\alpha(X)\,=\,\int K(X,\y)\,P_\alpha(\y)\,d\mu(\y),\qquad \forall \alpha\in\N^p,\]
that is, the $(P_\alpha)$ are eigenvectors of the associated operator on $L_2(\mu)$, and so
\[p(X)\,=\,\int K(X,\y)\,p(\y)\,d\mu(\y),\qquad \forall p\in\R[X].\]
The function $\x\mapsto \Lambda_{\mu,d}(\x):=\kappa_{\mu,d}(\x,\x)^{-1}$ is called the
{\it Christoffel function} associated with $\mu$ and $d\in\N$. The following result states a fundamental extremal property of the Christoffel function.

\begin{thm}[see e.g. \cite{dunkl2001orthogonal,nevai1986geza}]
		\label{th:christoffel}
		Let $\bxi\in\R^p$ be fixed, arbitrary. Then
		\begin{equation}
			\label{eq:christoffel}
			\Lambda_{\mu,d}(\bxi)\,=\,\min_{P\in\RR[X]_d}\:\left\{\int_{\RR^p} P(\x)^2\,d\mu(\x): P(\bxi)\,=\,1 \right\}.
		\end{equation}
		\end{thm}

		The Christoffel function plays an important role in orthogonal polynomials and the theory of interpolation and approximation, see \textit{e.g.} \cite{szego1974orthogonal,dunkl2001orthogonal}. One is particularly interested in the asymptotics of the normalized Christoffel function $\x\mapsto s(d)\Lambda_{\mu,d}(\x)$ as $d\to\infty$. The subject has a very long history in the univariate case, see \cite{nevai1986geza} for a detailed historical account prior to the 80's. The first quantitative asymptotic result was given in \cite{mate1980bernstein} and was latter improved by \cite{mate1991szego} and \cite{totik2000asymptotics}. In the multivariate setting, precise results are known in some particular cases such as balls, spheres and simplices \cite{bos1994asymptotics,bos1998asymptotics,xu1996asymptotics,xu1999asymptoticsSimplex,kroo2012christoffelBall} but much remains to be done for the general multivariate case. A typical example of asymptotic result is given under quite general (and technical) conditions in \cite{kroo2012christoffelBall,kroo2012christoffel}. This work
shows that, as $d\to\infty$, the limit  of the ratio, $\frac{\Lambda_{\nu,d}}{\Lambda_{\mu,d}}$, of two Christoffel functions associated to two mutually absolutely continuous measures $\mu$ and $\nu$, converges to the density $\frac{d\nu}{d\mu}(\x)$ on the interior of their common support. 

\begin{rem}
\label{numeval}
Notice that Theorem \ref{th:christoffel} also provides a method to compute the numerical value $\Lambda_{\mu,d}(\x)$ for $\x\in\RR^n$, fixed, arbitrary. Then indeed  
(\ref{eq:christoffel}) is a convex quadratic programming problem which can be solved efficiently, even in high dimension using first order methods such as projected gradient descent and its stochastic variants. This is particularly interesting when the  
nonsingular moment matrix $\M_d(\mu)$ is large.
\end{rem}

We next provide additional insights on Theorem \ref{th:christoffel} and solutions of (\ref{eq:christoffel}).

\begin{thm}
\label{th-christoffel}
For any $\bxi \in \RR^p$, the optimization problem in (\ref{eq:christoffel}) is convex with a unique optimal solution $P^*_d\in\RR[X]_d$ defined 
by
\begin{equation}
\label{pstar}
P^*_d(X)\,=\,\frac{\kappa_{\mu,d}(X,\bxi)}{\kappa_{\mu,d}(\bxi,\bxi)}\,=\,
\Lambda_{\mu,d}(\xi)\,\kappa_{\mu,d}(X,\bxi).
\end{equation}
In addition, 
\begin{eqnarray}
\label{th-christoffel-1}
\Lambda_{\mu,d}(\bxi)&=&\int P^*_d(\x)^2\,d\mu(\x)\,=\,\int P^*_d(\x)\,d\mu(\x)\\
\label{th-christoffel-2}
 \Lambda_{\mu,d}(\bxi) \bxi^\alpha &=&\int \x^\alpha\,P^*_d(\x)\,d\mu(\x)\qquad\alpha\in\N^p_d.
\end{eqnarray}
\end{thm}
The proof is postponed to Section \ref{appendix}. Interestingly, each of the orthonormal polynomials $(P_\alpha)_{\alpha\in\N^p}$ also satisfies an important and well-known extremality property. 
\begin{thm}[see e.g. \cite{dunkl2001orthogonal}]
\label{extremal2}
Let $\alpha\in\N^p$ be fixed, arbitrary and let $d=\vert\alpha\vert$. Then up to a multiplicative positive constant, $P_\alpha$ is the unique optimal solution of 
\begin{equation}
\label{th-extremal2-1}
\displaystyle\min_{P\in\R[\x]_d}\,\left\{\,\int P^2(\x)\,d\mu(\x):\: P(\x)=\x^\alpha+\sum_{\beta\lgl\, \alpha}\theta_\beta\,\x^\beta\,\quad\mbox{for some $\{\theta_\beta\}_{\beta \lgl \alpha}$}  \right\}.
\end{equation}
\end{thm}
Finally, we highlight the following important property which will be useful in the sequel.
\begin{thm}[See e.g. \cite{lasserre2016sorting}]
	\label{th:invariance}
	$\Lambda_{\mu,d}$ is invariant by change of polynomial basis $\v_d$, change of the origin of $\RR^p$ or change of basis in $\RR^p$.
\end{thm}
\begin{rem}
				All these statements can be deduced from identity \eqref{eq:defKappa}. Indeed, we have, for any $\bx,\by \in \RR^p$, 
				\begin{align*}
								\v_d(\x)^T \M_d(\mu)^{-1}\v_d(\y) &= (A\v_d(\x))^T \left( A \M_d(\mu)A^T\right)^{-1}(A\v_d(\y)) \\
								&= (A\v_d(\x))^T \left( \int_{\RR^p} (A\v_d(\z))(A\v_d(\z))^T d \mu(\z) \right)^{-1}(A\v_d(\y))
				\end{align*}
				for any invertible matrix $A$ of suitable size. All the proposed transformations induce a change of basis of polynomials up to degree $d$ which can be represented by such an $A$.
				\label{rem:changeBasis}
\end{rem}

\subsection{When $\mu$ is the Lebesgue measure}~
In this section we consider the important case of the Lebesgue measure on a compact set $S\subset\R^p$ such that ${\rm cl}({\rm int}(S)) = S$. It is known that in this case the Christoffel function encodes information on the set $S$; see for example the discussion in Section \ref{sec:christo}. In particular, the scaled Christoffel function remains positive on the interior of $S$. We push this idea further and present a new result asserting that it is possible to recover the set $S$ with strong asymptotic guaranties by carefully thresholding the corresponding scaled Christoffel function.
  
For any measurable set $A$, denote by $\mu_{A}$ the uniform probability measure on $A$, that is $\mu_A = \lambda_A / \lambda(A)$ where $\lambda$ is the Lebesgue measure and $\lambda_A$ the measure consisting of the restriction of Lebesgue measure to $A$ which is defined by $\lambda_A(A') = \lambda_{A \cap A'}$ for any measurable set $A'$.

\subsubsection*{Threshold and asymptotics}
The main idea is to use quantitative lower bounds on the scaled Christoffel function, $s(d)\Lambda_{\mu_S,d}$, on the interior of $S$ (Lemma \ref{lem:lowerBoundIn}) and upper bounds outside $S$ (Lemma \ref{lem:upperBoundOut}). Recall that $\mu_S$ denotes the uniform measure on $S$. In combining these bounds one proves the existence of a sequence of thresholds of the scaled Christoffel function which estimate $S$ in a strongly consistent manner. Let us introduce the following notation and assumption.
\begin{assumption}\hfill
	\label{ass:dk}
	\begin{itemize}
		\item[\it (a)] $S \subset \RR^p$ is a compact set such that ${\rm cl}({\rm int}(S)) = S$.
		\item[\it (b)] The sequence $\left( \delta_k \right)_{k \in \NN}$ is a decreasing sequence of positive numbers converging to $0$. For every $k \in \NN$,  let $d_k$ be the smallest integer such that:
			\begin{align}
				\label{eq:d_k}
				2^{3 - \frac{\delta_k d_k}{\delta_k + {\rm diam}(S)}} d_k^p \left( \frac{e}{p} \right)^p \exp\left( \frac{p^2}{d_k} \right) \leq \alpha_k 
\end{align}
	where ${\rm diam}(S)$ denotes the diameter of the set $S$, and
\[\alpha_k :=\frac{\delta_k^p\omega_p}{\lambda(S)}\frac{(d_k+1)(d_k+2)(d_k+3)}{(d_k+p+1)(d_k+p+2)(2d_k + p + 6)}\]
	\end{itemize}
\end{assumption}
\begin{rem}[On Assumption \ref{ass:dk}]\hfill
	\begin{itemize}
		\item $d_k$ is well defined. Indeed, since $\delta_k$ is positive, the left hand side of (\ref{eq:d_k}) goes to $0$ as $k\to \infty$ while the right hand side remains bounded for increasing values of $d_k$. 
		\item From the definition of $d_k$ and the fact that $\delta_k$ is decreasing, the sequence $\{d_k\}_{k\in\N}$ is non decreasing. Indeed, in (\ref{eq:d_k}) the right hand side is an increasing function of $\delta_k$ while the left hand side is decreasing so that if (\ref{eq:d_k}) is satisfied for a certain value of $d_k$ and $\delta_k$, it is also satisfied with the same $d_k$ and any value of $\delta \geq \delta_k$.
		\item Given $\left\{ \delta_k \right\}_{k \in \NN}$, computing $d_k$ can be done recursively and only requires the knowledge of ${\rm diam}(S)$ and $\lambda(S)$.
		\item A similar condition can be enforced if only upper bounds on ${\rm diam}(S)$ and on $\lambda(S)$ are available. In this case, replace these quantities by their upper bounds in (\ref{eq:d_k}) to obtain a similar result.
	\end{itemize}
\end{rem}
We are now ready to state the first main result of this section whose proof is postponed to Section \ref{appendix} for sake of clarity of exposition. Recall the definition of the Hausdorff distance $d_H(X,Y)$ between two subsets 
$X,Y$ of $\RR^p$:
\begin{align*}
	d_H(X,Y) = \max\left\{\sup_{\x\in X}\inf_{\y \in Y} {\rm dist}(\x, \y),  \sup_{\y\in Y}\inf_{\x \in X} {\rm dist}(\x, \y)\right\}.
\end{align*}
\begin{thm}
	\label{th:determinist}
	Let $S \subset \RR^p$, $\{\delta_k\}_{k \in \NN}$, $\{\alpha_k\}_{k \in \NN}$ and $\{d_k\}_{k \in \NN}$ satisfy Assumption \ref{ass:dk}. For every $k \in \NN$ let $S_k\subset\R^p$ be the set defined by,
	\begin{align*}
		S_k := \left\{ \x \in \RR^p:\:s(d_k)\,\Lambda_{\mu_S,d_k}(\x) \geq\alpha_k\right\}.
	\end{align*}
	Then, as $k \to \infty$,
	\begin{align*}
		d_H(S_k, S) &\to 0\\
		d_H(\partial S_k, \partial S) &\to 0.			
	\end{align*}
\end{thm}
\begin{rem}
				\label{rem:boundary}
				The relevance of Hausdorff distance and the notion of distance between topological boundaries is discussed in \cite{cuevas2006plugin} and \cite{singh2009adaptive}.
\end{rem}
\subsubsection*{Extension to more general probability measures}

Theorem \ref{th:determinist} can easily be extended to probability measures that are more general than uniform distributions, in which case we consider the following alternative assumption.
\begin{assumption}\hfill
	\label{ass:dkBis}
	\begin{itemize}
		\item[\it (a)] $S \subset \RR^p$ is a compact set such that ${\rm cl}({\rm int}(S)) = S$. 
		\item[\it (b)] The function $w\colon {\rm int}(S) \to [w_{-},+\infty)$ is integrable on ${\rm int}(S)$ with $w_{-} > 0$. The measure $\mu$ is such that for any measurable set $A$, $\mu(A) = \int_{A\cap S} w(\x) d\x$ and $\mu(S) = 1$. $\left\{ \delta_k \right\}_{k \in \NN}$ is a decreasing sequence of positive numbers which converges to $0$. For every $k \in \NN$, let  $d_k$ be the smallest integer such that:
			\begin{align}
				\label{eq:d_kBis}
				2^{3 - \frac{\delta_k d_k}{\delta_k + {\rm diam}(S)}} d_k^p \left( \frac{e}{p} \right)^p \exp\left( \frac{p^2}{d_k} \right) \leq \alpha_k 
						\end{align}
where
\[\alpha_k	:=w_-\delta_k^p\omega_p\frac{(d_k+1)(d_k+2)(d_k+3)}{(d_k+p+1)(d_k+p+2)(2d_k + p + 6)}.\]
	\end{itemize}
\end{assumption}
Under Assumption \ref{ass:dkBis} we obtain the following analogue of Theorem \ref{th:determinist}.

\begin{thm}
	\label{th:determinist2}
	Let $S \subset \RR^p$, $w \colon S \to [w_{-}, + \infty)$, $\{\delta_k\}_{k \in \NN}$, $\{\alpha_k\}_{k \in \NN}$ and $\{d_k\}_{k \in \NN}$ satisfy Assumption \ref{ass:dkBis}. For every $k \in \NN$, let
	\begin{align*}
		S_k := \left\{ \x \in \RR^p:\: s(d_k)\,\Lambda_{\mu_S,d_k}(\x) \geq\alpha_k\right\}.
	\end{align*}
	Then, as $k \to \infty$,
	\begin{align*}
		d_H(S_k, S) &\to 0\\
		d_H(\partial S_k, \partial S) &\to 0.			
	\end{align*}
\end{thm}
The proof is postponed to Section \ref{appendix}.
\subsection{Discrete approximation via the empirical Christoffel function}

In this section $\mu$ is a probability measure on $\R^p$ with compact support $S$. We focus on the statistical setting where information on $\mu$ is available only through a sample of points drawn independently from the given distribution $\mu$. In this setting, for every $n \in \NN$, let $\mu_n$ denote the empirical measure uniformly supported on an independent sample of $n$ points distributed according to $\mu$. It is worth emphasizing that  in principle $\Lambda_{\mu_n,d}$ is easy to compute and requires the inversion of a square matrix of size $s(d)$, see (\ref{eq:defKappa}). Note that the definition in (\ref{eq:defKappa}) can only be used if the empirical moment matrix, $\M_d(\mu_n)$ is invertible which is the case almost surely if $\M_d(\mu)$ is invertible and $n$ is large enough. Alternatively, the numerical evaluation of $\Lambda_{\mu_n,d}(\x)$ at $\x\in\R^n$, fixed arbitrary, reduces to solving the convex quadratic programming problem (\ref{eq:christoffel}), which can in principle be done efficiently even in high dimension, see Remark \ref{numeval}.

Our second main result is for fixed $d\in\NN$ and relates the population Christoffel function $\Lambda_{\mu,d}$ and its empirical version $\Lambda_{\mu_n,d}$, as $n$ increases. We proceed by distinguishing what happens far from $S$ and close to $S$. First, Lemma \ref{lem:prelimDiscrete} ensures that both Christoffel functions associated with $\mu$ and $\mu_n$ vanish far from $S$ so that the influence of this region can be neglected. Second, 
when closer to $S$ one remains in a compact set and the strong law of large numbers applies. 

\begin{thm}
	\label{th:discretization}
	Let $\mu$ be a probability measure on $\RR^p$ with compact support. Let $\left\{ X_i \right\}_{i \in \NN}$ be a sequence of i.i.d. $\R^p$-valued random variables with common distribution $\mu$. For $n = 1, 2, \ldots$, define the (random) empirical probability measure $\mu_n = \frac{1}{n}\sum_{i=1}^n \delta_{X_i}$. Then, for every $d \in \NN$, $d>0$, such that the moment matrix $\M_d(\mu)$ is invertible, it holds that
	\begin{align}	
		\label{eq:discretization}
		\sup_{\x \in \RR^p} \left\{|\Lambda_{\mu_n,d}(\x) - \Lambda_{\mu,d}(\x)|\right\} \quad \overset{a.s.}{\underset{n \to \infty}{\longrightarrow}} \quad0.
	\end{align}
	Equivalently
	\begin{align}	
		\label{eq:discretization22}
		\Vert\Lambda_{\mu_n,d}- \Lambda_{\mu,d}\Vert_\infty \quad \overset{a.s.}{\underset{n \to \infty}{\longrightarrow}} \quad0.
	\end{align}
	where $\Vert\cdot\Vert_\infty$ denotes the usual ``sup-norm".
\end{thm}
A detailed proof can be found in Section \ref{appendix}.
Theorem \ref{th:discretization} is a strong result which states a highly desirable property, namely that {\it almost surely}
with respect to the random draw of the sample, the (random) function $\Lambda_{\mu_n,d}(\cdot)$ converges to $\Lambda_{\mu,d}(\cdot)$ {\it uniformly in $\x\in\R^p$} as $n$ increases. Since we manipulate polynomials, it can be checked that \cite[Theorem 1]{cuevas2006plugin} for general level sets can be applied in the setting of Christoffel level set estimation. We get the following consequence in terms of consistency of the boundary of plugin estimates for Christoffel level sets.

\begin{thm}
	\label{th:discrLevelSet}
	Let $\mu$ be a probability measure on $\RR^p$ with compact support. Let $\left\{ X_i \right\}_{i \in \NN}$ be a sequence of i.i.d. $\R^p$-valued random variables with common distribution $\mu$. For $n = 1, 2, \ldots$, define the (random) empirical probability measure $\mu_n = \frac{1}{n}\sum_{i=1}^n \delta_{X_i}$. Then, for every $d \in \NN$, $d>0$, such that the moment matrix $\M_d(\mu)$ is invertible and any $c \in (0,\sup_{\x \in \RR^p} \{\Lambda_{\mu,d}(\bx)\})$, as $n$ increases, it holds that
	\begin{align}	
		\label{eq:discretization2}
		d_H(\partial L_n, \partial L)\quad \overset{a.s.}{\longrightarrow} \quad0,
	\end{align}
	where $L = \left\{ \x \in \RR^p, \Lambda_{\mu,d}(x) \geq c \right\}$ and $L_n = \left\{ \x \in \RR^p, \Lambda_{\mu_n,d}(x) \geq c \right\}$.
\end{thm}

\section{Applications}
\label{sec:applications}
\subsection{Rationale}
In this section we describe some applications for which properties of the Christoffel function prove to be very useful in a statistical context. We only consider the case of bounded support. A relevant property of the scaled Christoffel function is that it encodes information on the support and the density of a population measure $\mu$:
\begin{itemize}
	\item \cite{mate1991szego}, \cite{totik2000asymptotics} and \cite{kroo2012christoffel} provide asymptotic results involving the density of the input measure.
	\item Theorem \ref{th:determinist} provides asymptotic results related to the support of the input measure.
\end{itemize}
The support and density of a measure is of interest in many statistical applications. However, the aforementioned results are limited to population measures which are not accessible in a statistical setting. In the context of empirical Christoffel functions, Theorem \ref{th:discretization} suggests that these properties still hold (at least in the limit of large number of samples) when one uses the empirical measure $\mu_n$ in place of the population measure $\mu$. Combining these ideas suggests to use of the empirical Chritoffel function in statistical applications such as (a) density estimation, (b) support inference or (c) outlier detection. This is illustrated on simulated and real world data and we compare the performance with well established methods for the same purpose. Finally, we also describe another application, namely {\it inversion of affine shuffling}, whose links with statistics are less clear.

All results presented in this section are mainly for illustrative purposes. In particular, the choice of the degree $d$ as a function of the sample size $n$ was done empirically and a precise quantitative analysis is a topic of future research beyond the scope of the present paper.
\subsection{Density estimation}
\label{sec:applDensity}
\begin{figure}[t]
	\centering
	\includegraphics[width=\textwidth]{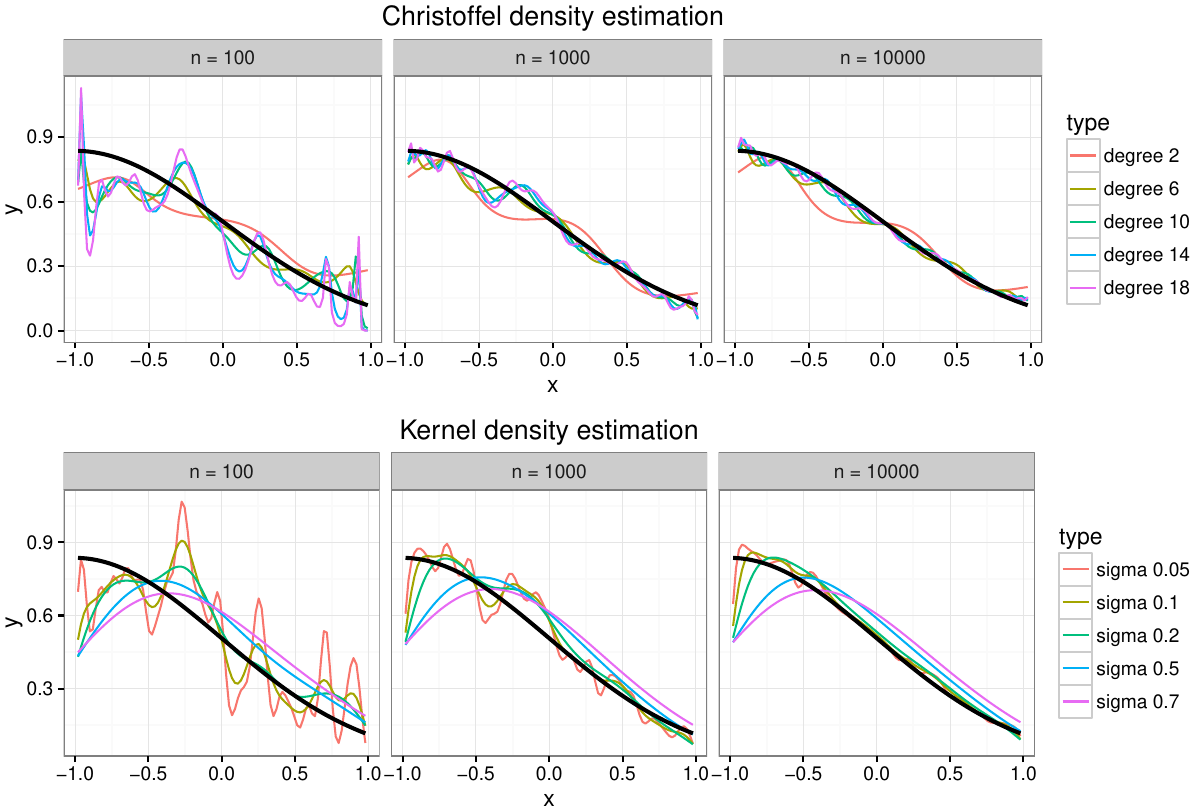}
	\caption{Comparison of Christoffel and kernel density estimation with Gaussian kernel. The same samples are used in both cases. We vary the sample size $n$, the degree $d$ for the Christoffel function and the scale parameter $\sigma$ for the Gaussian kernel. The black curve shows the population density.}
	\label{fig:densityEst}
\end{figure}
Most asymptotic results regarding the scaled Christoffel function suggest that the limiting behaviour involves the product of a boundary effect term and a density term. Hence if one knows both the Christoffel function and the boundary effect term, one has access to the density term. Unfortunately, this boundary term is only known in specific situations, the most typical example being the Euclidean ball. Hence, in the present state of knowledge, one of the following is assumed to hold true.
\begin{itemize}
	\item The support of the population measure $\mu$ is $S$ and $\lim_{d\to \infty}s(d)\Lambda_{\lambda_S,d}$ exists (possibly unknown).
	\item The support is unknown but contains a set $S$ with the same property as above. In this case, we consider the restriction of the population measure $\mu$ to $S$. Note that a sample from the restriction is easily obtained from a sample from $\mu$ by rejection.
\end{itemize}
In both cases, assuming that $\mu$ has a density $h$ on $S$, it is expected that the ratii $\frac{\Lambda_{\mu,d}}{\Lambda_{\lambda_S,d}}$ or $\frac{s(d)\Lambda_{\mu,d}}{\lim_{d\to\infty} s(d)\Lambda_{\lambda_S,d}}$, converge to $h$. An example of such a result in the univariate setting is the following.

\begin{thm}[Theorem 5 \cite{mate1991szego}]
	\label{th:density}
	Suppose that $\mu$ is supported on $[-1,1]$ with density $h\geq a>0$. Then for almost every $x \in [-1,1]$,
	\begin{align*}
		\lim_{d \to \infty}d \Lambda_{\mu,d}(x) \to \pi h(x) \sqrt{1 - x^2}.
	\end{align*}
\end{thm}

Extensions include \cite{totik2000asymptotics} for general support and \cite{kroo2012christoffel} for the multivariate setting. Combining Theorems \ref{th:density} and \ref{th:discretization} suggest that the empirical Christoffel function can be used for density estimation. For illustration purposes, we set $\mu$ to be the restriction of a Gaussian to $[-1,1]$. We perform the following experiment for given $n,d \in \NN$.
\begin{itemize}
	\item Generate $x_1, \ldots, x_n \in [-1,1]$ sampled independently  from $\mu$.
	\item Compute and plot $x \to \frac{\Lambda_{\mu_n,d}(x)}{\Lambda_{\lambda_S,d}(x)}$. Note that $\Lambda_{\lambda_S,d}$ is easily derived from the moments of the uniform distribution on $[-1,1]$.
\end{itemize}
The result is presented in Figure \ref{fig:densityEst} and a comparison is given with a classical technique, kernel density estimation \cite{rosenblatt1956remark,parzen1962estimation} with the Gaussian kernel. The result suggest that empirical Christoffel based density estimation is competitive with kernel density estimation in this setting. It is worth noticing how simple the methodology is with a single parameter to tune.

\subsection{Support inference}
\label{sec:applSupport}
\begin{figure}[t]
	\centering
	\includegraphics[width=.8\textwidth]{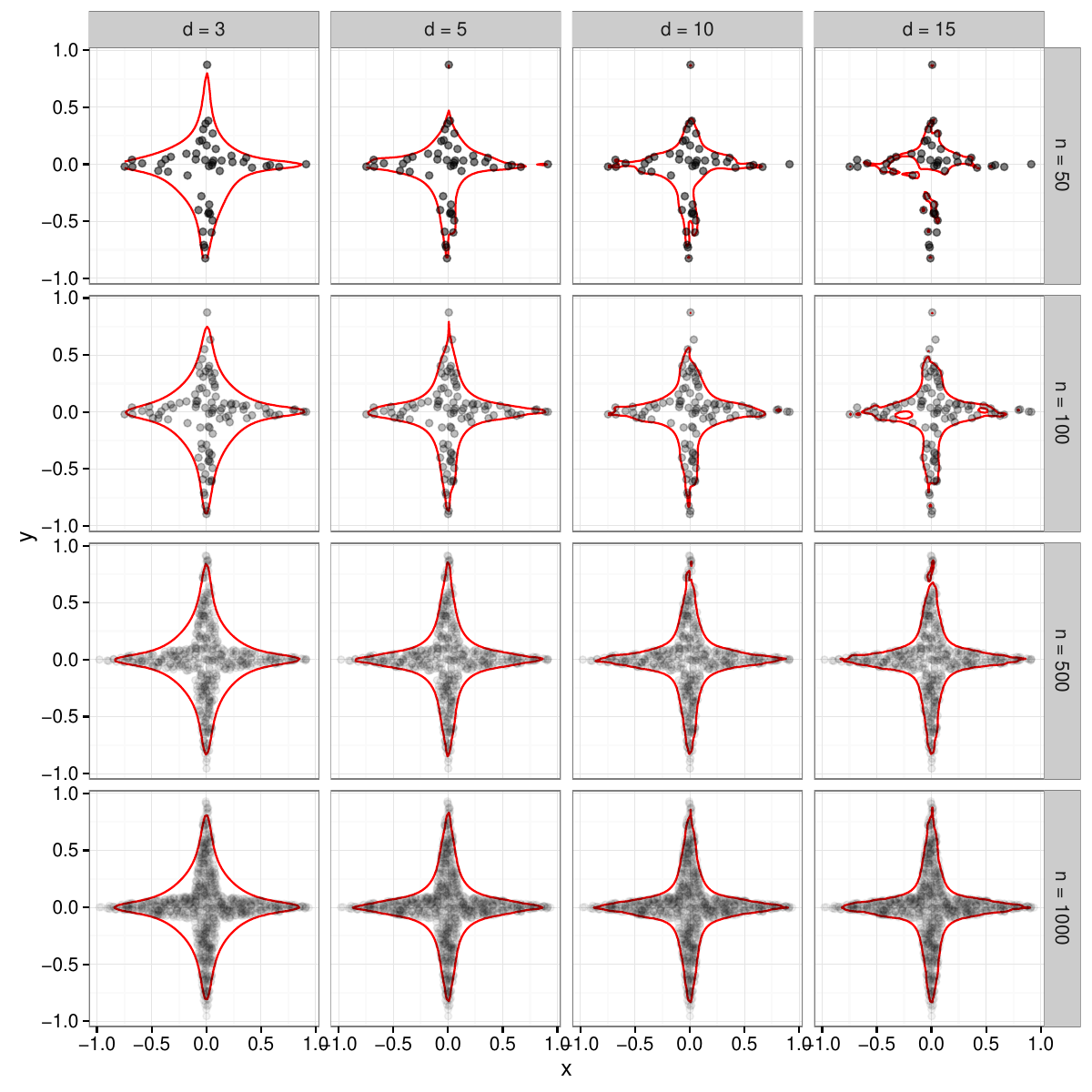}
	\caption{Finie sample from the uniform measure over a star shaped domain in $\RR^2$. For each value of $n$ and $d$, the red line represents the sublevel set $\left\{\x \in \RR^2,\,s(d)\Lambda_{\mu_n,d}(\x) = \alpha(\delta) \right\}$, where $\delta = 0.5$ and $\alpha$ is given in Assumption \ref{ass:dk}.}
	\label{fig:suppInf}
\end{figure}
Combining Theorems \ref{th:determinist} and \ref{th:discretization} suggest that one may recover the unknown support of a population measure $\mu$ from $n$ independant samples by thresholding the scaled empirical Christoffel function. In this section we set $\mu$ to be the uniform probability measure over a star shaped domain in $\RR^2$ (see Figure \ref{fig:suppInf}). For different values of the degree $d$ and sample size $n$, we plot in Figure \ref{fig:suppInf} the corresponding sample and the associated level set
$\left\{\x \in \RR^2:\,s(d)\Lambda_{\mu_n,d}(\x) = \alpha(\delta) \right\}$, where $\delta = 0.5$ and $\alpha$ is given in Assumption \ref{ass:dk}. 

The results displayed in Figure \ref{fig:suppInf} show that for well chosen values of $d$ and with $\alpha$ 
as in Assumption \ref{ass:dk}, the support 
of the population measure is rather well approximated from a finite independent sample. The results even suggest that a careful tuning of the degree $d$ and the threshold level set $\alpha$ allows to approximate the support extremely well for larger sample sizes. Of course the degree $d$ should be chosen to avoid a form of over-fitting as the results suggest for small sample sizes and large values of $d$. A precise analysis of this phenomenon is a topic of future research.

\subsection{Outlier detection}
\label{sec:applOutlier}

\begin{table}
	\centering
	\begin{tabular}{|c|cccccc|}
		\hline
		\multicolumn{7}{|c|}{\bf 1SVM}\\\hline
		\diagbox{$\sigma$}{$\nu$} &0.005&0.01&0.02&0.05&0.1&0.2\\\hline
		0.01	&10&17&17				&17				&15				&11				\\
		0.02	&2&17	&{\bf 18}	&17				&15				&12				\\
		0.05	&8&1 	&14				&{\bf 18}	&15				&11				\\
		0.1		&9&8	&12				&17				&14				&11				\\
		0.2		&7&9	&8				&17				&14				&13				\\
		0.5		&3&5	&9				&15				&17				&16				\\
		1			&3&6	&9				&14				&{\bf 19}	&{\bf 18}	\\
		2			&4&4	&5				&1				&{\bf 18}	&{\bf 18}	\\
		5			&4&3	&4				&9				&12				&16				\\
		10		&5&4	&4				&7				&10				&15				\\\hline
	\end{tabular}
	\;
	\begin{tabular}{|c|c|}
		\hline
		\multicolumn{2}{|c|}{\bf Christoffel}\\\hline
		$d$&AUPR\\\hline
		1&8\\
		2&{\bf 18}\\
		3&{\bf 18}\\
		4&16\\
		5&15\\
		6&13\\\hline
	\end{tabular}
	\;
	\begin{tabular}{|c|c|}
		\hline
		\multicolumn{2}{|c|}{\bf KDE}\\\hline
		$\sigma$&AUPR\\\hline
		0.01&8\\
		0.02&1\\
		0.05&13\\
		0.1&13\\
		0.2&12\\
		0.5&5\\
		1&4\\\hline
	\end{tabular}
	\vspace{.1in}
	\caption{AUPR scores ($\times 100$) for the network intrusion detection task for the three different methods considered in this paper. 1SVM corresponds to one-class SVM with Gaussian kernel and varying kernel scale parameter $\sigma$ and SVM parameter $\nu$. Christoffel corresponds to the empirical Christoffel function with varying degree $d$. KDE corresponds to kernel density estimation with Gaussian kernel and varying scale parameter $\sigma$. The best scores are higlighted in boldface font.}
	\label{tab:outlierDetection}
\end{table}

~In \cite{lasserre2016sorting} we suggested 
that the empirical Christoffel function could be used for the purpose of detecting outliers and the claim was supported by some numerical experiments. The rationale for this is that the empirical Christoffel function encodes information about the population density and outliers can be seen as samples from low density areas. We follow the same line and consider the network intrusion detection task described in \cite{williams2002comparative} based on the KDD cup 99 dataset \cite{lichmanUCI2013}. Following the pre-processing described in \cite{williams2002comparative, lasserre2016sorting}, we build up five datasets consisting of network connections represented by labeled vectors in $\RR^3$ where each label indicates wether the connection was an attack or not.
\begin{center}
	\begin{tabular}{c|c|c|c|c|c}
		Dataset&http&smtp&ftp-data&ftp&other\\\hline
		Number of examples&567498&95156&30464&4091&5858\\
		Proportions of attacks&0.004&0.0003&0.023&0.077&0.016
	\end{tabular}
\end{center}
All the experiments described in this section are performed on the ``other'' dataset which is the most heterogeneous. The main task is to recover attacks from the collection of points in $\RR^3$, ignoring the labels, and then compare the predictions with the ground thruth (given by the labels). We compare different methods, each of them assign a score to an individual, the higher the score, the more likely the individual is to be an outlier, or an attack. The metric that we use to compare different methods is the area under the Precision Recall curve (AUPR); see for example \cite{davis2006relationship}. We compare three different methods, each of them producing a score reflecting some degree of outlyingness.
\begin{itemize}
	\item Empirical Christoffel function.
	\item Kernel density estimation \cite{rosenblatt1956remark,parzen1962estimation} with Gaussian kernel. The value of the density estimated at each datapoint is used as an outlyingness score.
	\item One-class SVM \cite{scholkopf2001estimating} with Gaussian kernel. The value of the estimated decision function at each datapoint is used as an outlyingness score. We used the implementation provided in the \texttt{kernlab} package \cite{zeileis2004kernlab}.
\end{itemize}
The first two methods involve only a single parameter while the last method requires two parameters to be tuned. The results are given in Table \ref{tab:outlierDetection} and the corresponding curves can be found in Section \ref{sec:PROutlier}. Table \ref{tab:outlierDetection} suggest that one-class SVM and the empirical Christoffel perform similarly and clearly outperform the kernel density estimation approach. It is worth noticing here that the one-class SVM provides slightly better performances but requires a precise tuning of the second parameter. 

\subsection{Inversion of affine shuffling}
\begin{figure}[t]
	\centering
	\includegraphics[width=.6\textwidth]{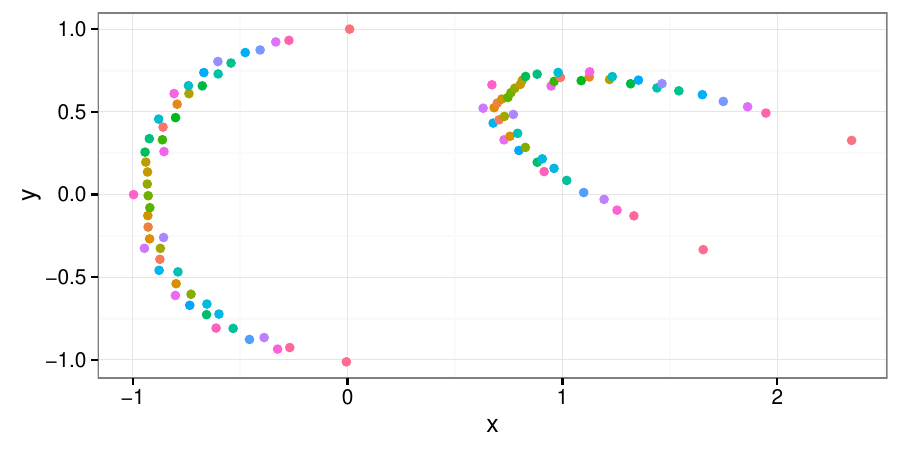}
	\includegraphics[width=.3\textwidth]{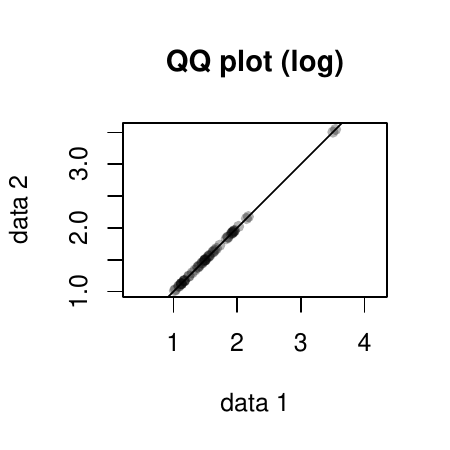}
	\caption{Illustration of the affine matching procedure based on the empircal Christoffel function. On the left, two datasets of points in $\RR^2$, the second one is the image of the first one by an affine transformation. The colours indicates the correspondance between the two clouds of points which have been recovered by matching the corresponding empirical Christoffel functions. The matching of these values is illustrated on the left with a quantile quantile plot of the empirical Chritoffel function values for each datasets. We applied a $\log$ transformation for readability and the first diagonal is represented.}
	\label{fig:affineMatching}
\end{figure}
This last application has fewer connections with statistics. Suppose that we are given two matrices $\X \in \RR^{p \times n}$ and $\X' \in \RR^{p \times n}$. Furthermore, we know that there exists an invertible affine mapping $\mathcal{A} \colon \RR^p \to \RR^p$ such that after a potential permutation of the columns, $\mathcal{A}$ defines a bijection between the columns of $\X'$ and those of $\X$. The problem is to recover the correspondence between the columns of $\X$ and the columns of $\X'$, whence the name ``affine shuffling inversion''. Note that the columns may be shuffled in an arbitrary way and therefore the matching problem is not trivial. In this setting we can use the affine invariance property of the Christoffel function described in Theorem \ref{th:invariance}. This is based on the two following observations.
\begin{itemize}
	\item The Christoffel function only depends on the empirical moments and hence is not sensitive to reshufling of the columns.
	\item Working with the affine image amounts to perform a change of basis and a change of origin. By Theorem \ref{th:invariance}, the evaluation of the Christoffel function does not change.
\end{itemize}
This suggests the following procedure.
\begin{itemize}
	\item Compute $\Lambda_{\X,d}$ and $\Lambda_{\X',d}$ the Christoffel functions associated to the columns of $\X$  and the columns of $\X'$ respectively.
	\item Set $A \in \RR^n$ to be the vector with $\Lambda_{\X,d}(\X_i)$ as $i$-th entry where $\X_i$ is the $i$-th column of $\X$. Set $A'$ similarly.
	\item Match the values in $A$ to the values in $A'$ according to their rank.
\end{itemize}
The proposed procedure defines a unique permutation between columns of $\X$ and columns of $\X'$ when there are no ties in the vectors $A$ and $A'$. In this case, Theorem \ref{th:invariance} ensures that we have found the correct correspondance. In case of ties, the procedure does not allow to elicit completely the correspondance matching. Overall, the method is not garanteed to work but allows to treat simple cases easily. Investigating the robustness of this procedure to noise or to matching mispecification is the subject of future research.

An illustration is given in Figure \ref{fig:affineMatching} where a moon shaped cloud of points in $\RR^2$ is deformed by an affine transformation and the matching between the points between the two clouds is recovered by matching the corresponding Christoffel function values. The correspondance between Christoffel function values is illustrated on a quantile quantile plot.

\section{Conclusion}
In this paper we have investigated the potential of the empirical Christoffel function for some applications in statistics and machine learning. This question led us to investigate its theoretical properties as well as potential paths toward applications, mostly in a statistical framework.

On the theoretical side, we proposed two main contributions. The first one provides an explicit thresholding scheme which allows to use the Christoffel function to recover the support of a measure with strong asymptotic guarantees. Although this property finds its root in the long history of results regarding asymptotic properties of the Christoffel function, we have provided a systematic way to tune the threshold and the degree to ensure strong convergence guarantees. The second main contribution relates the empirical Christoffel function to its population counterpart in the limit of large samples. This type of results is new and paves the way toward a much more precise understanding of relations between these two objects in a small sample setting.

On the practical side, we have illustrated the relevance of the Christoffel function 
as a practical tool in a machine learning context. In particular, simulations and experiments on real world data support our claim that the empirical Christoffel function is potentially useful for density estimation, support inference and outlier detection. Finally, in another application outside the statistical framework (detection of affine matching between two clouds of points), we have again illustrated the potential of the Christoffel function as a tool in shape recognition and shape comparison. 

Both theory and applications suggest a broad research program. As already mentioned, an important issue is to quantify the deviation of the empirical Christofel function from its population counterpart in a finite sample setting. Results in this direction could have both theoretical and practical impacts and would  to compare more accurately the performance of Christoffel-based approaches with state-of-the art methods. Furthermore, the use of the Christoffel function in a statistical framework raises questions specific to each application considered in this paper and will be the subject of future investigations. Finally, there are still important open questions on the  Christoffel function itself and works in the line of \cite{berman2009bergman} are of great interest to address applications in statistics.

\section*{Acknowledgements}
The research of the first author was funded by the European Research Council (ERC) under the European's Union Horizon 2020 research and innovation program (grant agreement 666981 TAMING).
\section{Appendix}
\label{appendix}

\subsection{Precision recall curves from section \ref{sec:applOutlier}}
\label{sec:PROutlier}
This section displays the curves from which the AUPR scores were measured in Section \ref{sec:applOutlier}. Christoffel function and kernel density estimation are presented in Figure \ref{fig:PRChristo} and the one-class SVM is presented in Figure \ref{fig:PROneClass}. A detailed discussion the experiment is given in Section \ref{sec:applOutlier}.
\begin{figure}[h!]
	\centering
	\includegraphics[width=.45\textwidth]{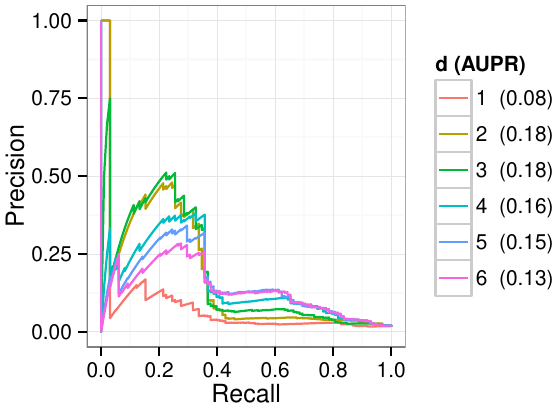}
	\includegraphics[width=.45\textwidth]{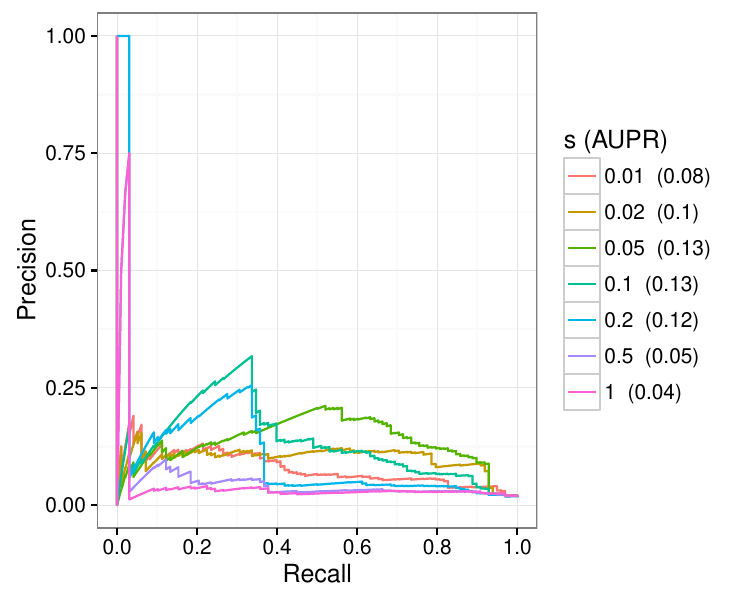}
	\caption{Precision recall curves for the network intrusion detection task. Left: Christoffel function with varying degree $d$. Right: kernel density estimation with Gaussian kernel and varying scale parameter $\sigma$.}
	\label{fig:PRChristo}
\end{figure}
\begin{figure}[h!]
	\centering
	\includegraphics[width=.9\textwidth]{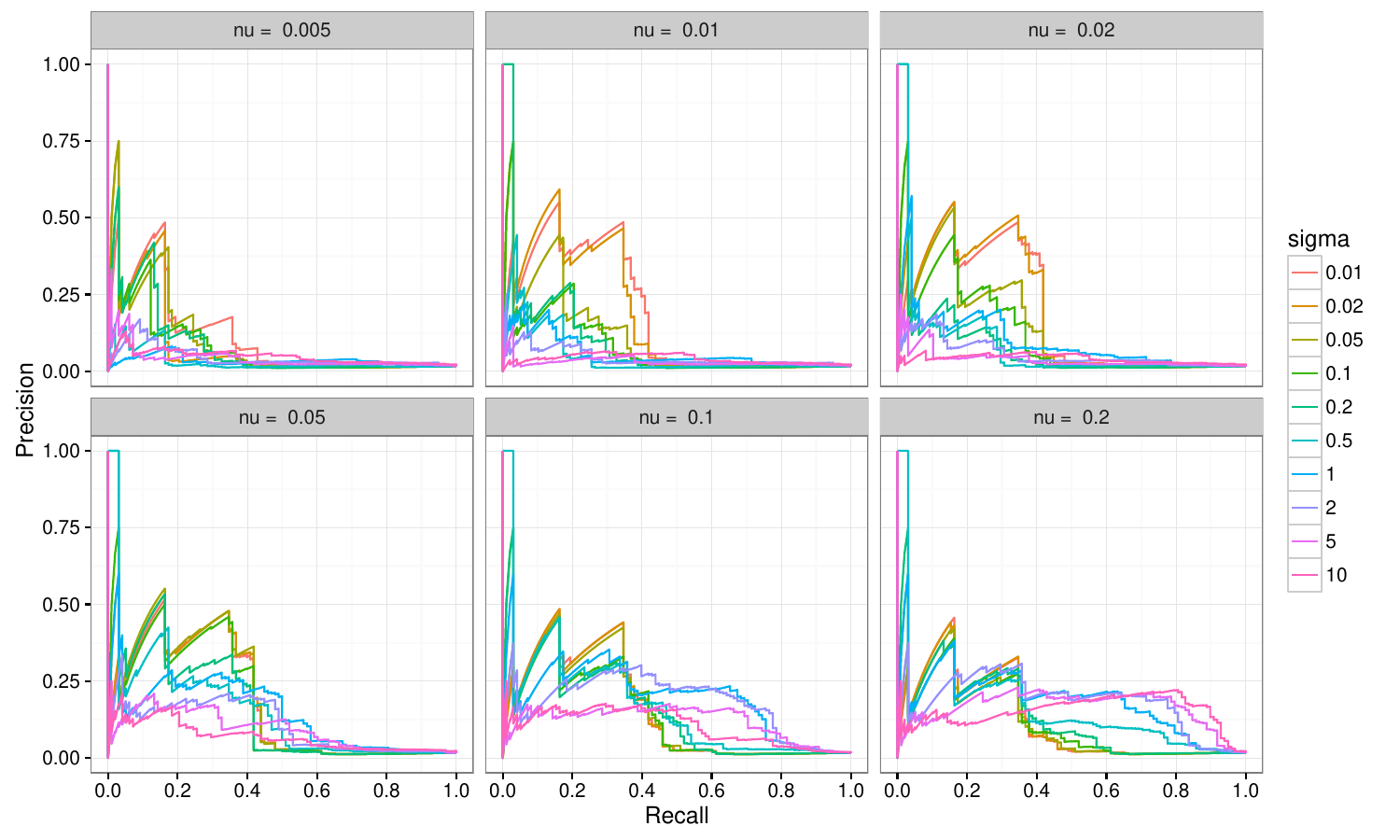}
	\caption{Precision recall curves for the network intrusion detection task. The method used is the one-class SVM with a Gaussian kernel. We vary the scale parameter $\sigma$ and the SVM parameter $\nu$. We used the SVM solver of the package \cite{zeileis2004kernlab}.}
	\label{fig:PROneClass}
\end{figure}

\subsection{Proof of Theorem \ref{th-christoffel}}
 
\begin{proof}
In the optimization problem (\ref{eq:christoffel}) the objective function
$P\mapsto \int P^2d\mu$ is strongly convex in the vector of coefficients of $P$ because
\[\int P^2\,d\mu\,=\,P^T \M_d(\mu)\,P\quad\mbox{and}\quad \M_d(\mu)\succ0,\]
and therefore (\ref{eq:christoffel}) reads $\min \{P^T\M_d(\mu)\,P: P^T\v_d(\xi)=1\}$, which is a convex optimization problem with a strongly convex objective function.  Slater's condition holds (only one linear equality constraint)
and so the Karush-Kuhn-Tucker (KKT) -optimality conditions are both necessary and sufficient. 
At an optimality solution $P^*_d$ they read:
\[P^*_d(\bxi)=1;\quad 2\M_d(\mu)\,P^*_d\,=\,\theta\, \v_d(\xi),\]
for some scalar $\theta$. Multiplying by $(P^*_d)^T$ yields
\[2\kappa_{\mu,d}(\bxi,\bxi)^{-1}\,=\,2(P^*_d)^T\M_d(\mu)P^*_d\,=\,\theta.\]
Hence necessarily
\[P^*_d(X)\,=\,\v_d(X)^T\,P^*_d\,=\,\frac{\theta}{2}\,\v_d(X)^T\M_d(\mu)^{-1}\,\v_d(\bxi)\,=\,
\frac{\kappa_{\mu,d}(X,\bxi)}{\kappa_{\mu,d}(\bxi,\bxi)},\]
which is (\ref{pstar}).
Next, let $\e_\alpha\in\RR^{s(d)}$ be the vector with null coordinates except the entry $\alpha$ which is $1$. From the definition of the moment matrix $\M_d(\mu)$,
\[\e_\alpha^T\M_d(\mu)\,P^*_d\,=\,
\int \x^\alpha P^*(\z)\,d\mu(\z)\,=\,\kappa_{\mu,d}(\bxi,\bxi)^{-1}\e_\alpha^T\v_d(\bxi)
\,=\,\kappa_{\mu,d}(\bxi,\bxi)^{-1}\,\bxi^\alpha,\]
which is (\ref{th-christoffel-2}). In particular with $\alpha:=0$, we recover (\ref{th-christoffel-1}),
\[\int P^*_d(\x)\,d\mu(\x)\,=\,\kappa_{\mu,d}(\bxi,\bxi)^{-1}\,=\,\int P^*_d(\x)^2\,d\mu(\x).\]
\end{proof}

\subsection{Proof of Theorems \ref{th:determinist} and \ref{th:determinist2}}

\subsubsection{Lower bound on the Christoffel function inside $S$}

We will heavily rely on results from \cite{bos1994asymptotics} (note that similar results could be obtained on the box, see for example \cite{xu1995christoffel}). In particular, we have the following result.
\begin{lem}
	\label{lem:lowerBoundInPrelim}
	We have for any $d \geq 2$ 
	\begin{align*}
		\frac{\kappa_{\lambda_\B,d}(0,0)}{s(d)} \leq\frac{1}{\omega_p} \frac{(d+p+1)(d+p+2)}{(d+1)(d+2)}\left( 1 + \frac{d+p+3}{d+3} \right)
	\end{align*}
\end{lem}
\begin{proof}
	Combining Lemma 2 in \cite{bos1994asymptotics} and the last equation of the proof of Lemma 3 in \cite{bos1994asymptotics}, we have 
	\begin{align*}
		\kappa_{\lambda_\B,d}(0,0) \leq \frac{1}{\omega_p}\left( {p+d+3 \choose p} + {p+d+2 \choose p} \right).	
	\end{align*}
	The result follows by using the expression given for $s(d)$ and simplifying factorial terms.
\end{proof}

From this result, we deduce the following bound. 
\begin{lem}
	\label{lem:lowerBoundIn}
	Let $\delta > 0$ and $\x \in S$ such that ${\rm dist}(\x, \partial S) \geq \delta$. Then
	\begin{align*}
		s(d)\Lambda_{\mu_S,d}(\x) \geq \frac{\delta^p\omega_p}{\lambda(S)}\frac{(d+1)(d+2)(d+3)}{(d+p+1)(d+p+2)(2d + p + 6)}.
	\end{align*}
\end{lem}
\begin{proof}
	Decompose  the measure $\mu_S$ into the sum,
	\begin{align*}
		\mu_S &= \frac{\lambda(S\setminus \B_{\delta}(\x))}{\lambda(S)} \mu_{S \setminus \B_{\delta}(\x)} + \frac{\lambda(\B_{\delta}(\x))}{\lambda(S)} \mu_{\B_{\delta}(\x)}. 
	\end{align*}
	Hence, by monotonicity of the Christoffel function with respect to addition and closure under multiplication by a positive term (this follows directly from Theorem \ref{th:christoffel}), we have
	\begin{align}
		\label{eq:lowerBoundInTemp1}
		\Lambda_{\mu_S,d}(\x) \geq \frac{\lambda(\B_{\delta}(\x))}{\lambda(S)} \Lambda_{\mu_{\B_{\delta}(\x)},d}(\x).
	\end{align}
	Next, by affine invariance of the Christoffel function (Theorem \ref{th:invariance}), 
	\begin{align}
		\label{eq:lowerBoundInTemp2}
		\Lambda_{\mu_{\B_{\delta}(\x)},d}(\x) = \Lambda_{\mu_{\B},d}(0) = \frac{1}{\lambda(\B)} \Lambda_{\lambda_\B,d}(0)=\frac{1}{\lambda(\B)} \frac{1}{\kappa_{\lambda_\B,d}(0,0)} ,
	\end{align}
	where $\B$ is the unit Euclidean ball in $\RR^p$. The result follows by combining (\ref{eq:lowerBoundInTemp1}), (\ref{eq:lowerBoundInTemp2}), Lemma \ref{lem:lowerBoundInPrelim} and the fact that $\frac{\lambda(\B_{\delta}(\x))}{\lambda(\B)} = \delta^p$.
\end{proof}

\subsubsection{Upper bound on the Christoffel function outside $S$}

We next exhibit an upper bound on the Christoffel function outside of $S$. We first provide a useful quantitative refinement of the ``Needle polynomial'' introduced in \cite{kroo2012christoffel}. 
\begin{lem}	
	\label{lem:needlePoly}
	For any $d \in \NN$, $d > 0$, and any $\delta \in (0,1)$, there exists a $p$-variate polynomial of degree $2d$, $q$, such tha
	\[q(\mathbf{0}) \,=\, 1\,;\quad 	-1\,\leq \,q \,\leq\, 1,\text{ on } \B\,;\quad 
		\vert q\vert\,\leq\, 2^{1-\delta d} \text{ on } \B\setminus \B_{\delta}(\x).\]
	
\end{lem}
\begin{proof}
	Let $r$ be the univariate polynomial of degree $2d$, defined by
	\begin{align*}
		r\colon t \to \frac{T_d(1+\delta^2 - t^2)}{T_d(1+\delta^2)},
	\end{align*}
	where $T_d$ is the Chebyshev polynomial of the first kind. We have 
	\begin{align}
		\label{eq:needlePolyTemp1}
		r(0) = 1. 
	\end{align}
	Furthermore, for $t \in [-1,1]$, we have $0 \leq 1+\delta^2 - t^2 \leq 1+\delta^2$. $T_d$ has absolute value less than $1$ on $[-1,1]$ and is inceasing on $[1, \infty)$ with $T_d(1) = 1$, so for $t \in [-1,1]$,
	\begin{align}
		\label{eq:needlePolyTemp2}
		-1 \leq r(t) \leq 1.
	\end{align}
	For $|t| \in [\delta, 1]$, we have $\delta^2 \leq 1+\delta^2-t^2 \leq 1$, so
	\begin{align}
		\label{eq:needlePolyTemp3}
		|r(t)| \leq \frac{1}{T_d(1+\delta^2)}.
	\end{align}
	Let us bound the last quantity. Recall that for $t \geq 1$, we have the following explicit expression 
	\begin{align*}
		T_d(t) = \frac{1}{2}\left( \left( t + \sqrt{t^2-1} \right)^d +  \left( t + \sqrt{t^2-1} \right)^{-d}\right).
	\end{align*}
	We have $1 + \delta^2 +\sqrt{(1+\delta^2)^2 - 1} \geq 1 + \sqrt{2} \delta$, which leads to
	\begin{align}
		\label{eq:needlePolyTemp4}
		T_d(1+\delta^2) &\geq \frac{1}{2}\left( 1 + \sqrt{2} \delta \right)^d\\
		&= \frac{1}{2} \exp\left( \log\left( 1+\sqrt{2}\delta \right) d \right)\nonumber\\
		&\geq\frac{1}{2} \exp\left( \log(1+\sqrt{2}) \delta d \right)\nonumber\\
		&\geq2^{\delta d - 1}, \nonumber
	\end{align}
	where we have used concavity of the $\log$ and the fact that $1+\sqrt{2} \geq 2$. It follows by combining (\ref{eq:needlePolyTemp1}), (\ref{eq:needlePolyTemp2}), (\ref{eq:needlePolyTemp3}) and (\ref{eq:needlePolyTemp4}), that $q \colon y \to r(\|y - x\|_2)$ satisfies the claimed properties.
\end{proof}

We recall the following well known bound for the factorial taken from \cite{robbins1955remark}.
\begin{lem}[\cite{robbins1955remark}]
	\label{lem:robbins}
	For any $n \in \NN$, we have
	\begin{align*}
		\exp\left( \frac{1}{12n + 1} \right) \leq \frac{n!}{\sqrt{2\pi n} n^n \exp(-n)} \leq \exp\left( \frac{1}{12n} \right).
	\end{align*}
\end{lem}
We deduce the following Lemma.
\begin{lem}
	\label{lem:Nd}
	For any $d \in \NN$, $d > 0$, we have 
	\begin{align*}
		{p+d \choose d} &\leq d^p \left( \frac{e}{p} \right)^p \exp\left( \frac{p^2}{d} \right) 
	\end{align*}
\end{lem}
\begin{proof}
	This follows from a direct computation using Lemma \ref{lem:robbins}.
	\begin{align*}
		{p+d \choose d} &= \frac{(p+d)!}{p!d!}\\
		&\leq \frac{\exp\left(\frac{1}{24} \right)}{\sqrt{2\pi}} \sqrt{\frac{p+d}{pd}} \frac{(p+d)^{p+d}}{p^pd^d} \\  
		&\leq \frac{\exp\left(\frac{1}{24} \right)}{\sqrt{2\pi}} \sqrt{2} \frac{d^p}{p^p} \left( 1 + \frac{p}{d} \right)^{p+d} \\  
		&\leq  \frac{d^p}{p^p} \exp\left( \frac{p^2}{d} + p \right)  
	\end{align*}
	which proves the result.
\end{proof}
Combining the last two Lemma, we get the following bound on the Christoffel function.
\begin{lem}
	\label{lem:upperBoundOut}
	Let $\x \not \in S$ and $\delta$ be such that ${\rm dist}(\x, S) \geq \delta$. Then, for any $d \in \NN$, $d> 0$, we have
	\begin{align*}
		s(d)\Lambda_{\mu_S, d}(\x) \leq 2^{3 - \frac{\delta d}{\delta + {\rm diam}(S)}} d^p \left( \frac{e}{p} \right)^p \exp\left( \frac{p^2}{d} \right).
	\end{align*}
\end{lem}
\begin{proof}
				We may translate the origin of $\RR^p$ at $\x$ and scale the coordinates by $\delta + {\rm diam}(S)$, this results in $\x= 0$ and distance from $\x$ to $S$ is at most $\delta' = \frac{\delta}{\delta + {\rm diam}(S)} \leq 1$. Furthermore, $S$ is contained in the unit Euclidean ball $\B$. Using invariance of the Christoffel function with respect to change of origin and change of basis in $\RR^p$, (Theorem \ref{th:invariance}), this affine transformation does not change the value of the Christoffel function. Now the polynomial described in Lemma \ref{lem:needlePoly} provides an upper bound on the Christoffel function. Indeed for any $d' \in \NN$, we have
	\begin{align}
		\label{eq:upperBoundOutTemp2}
		\Lambda_{\mu_{S},2d' + 1}(0) \leq \Lambda_{\mu_{S},2d'}(0) \leq 2^{2 - 2\delta' d'} \leq 2^{3 - \delta'(2d' + 1)}, 
	\end{align}
	where we have used $\delta' \leq 1$ to obtain the last inequality. Combining Lemma \ref{lem:Nd} and (\ref{eq:upperBoundOutTemp2}), we obtain for any $d' \in \NN$
	\begin{align}
		\label{eq:upperBoundOutTemp3}
			s(2d')\,\Lambda_{\mu_{S},2d'}(0) &\leq 2^{3 - 2\delta' d'} (2d')^p \left( \frac{e}{p} \right)^p \exp\left( \frac{p^2}{2d'} \right), \\
			s(2d'+1)\,\Lambda_{\mu_{S},2d' + 1}(0) &\leq 2^{3 - \delta'(2d' + 1)}(2d' + 1)^p \left( \frac{e}{p} \right)^p \exp\left( \frac{p^2}{2d'+1} \right)\nonumber. 
	\end{align}
	Since in (\ref{eq:upperBoundOutTemp3}) $d'\in\NN$ was arbitrary, we obtain in particular 
	\begin{align}
		\label{eq:upperBoundOutTemp4}
			s(d)\,\Lambda_{\mu_{S},d}(0) &\leq 2^{3 - \delta' d} d^p \left( \frac{e}{p} \right)^p \exp\left( \frac{p^2}{d} \right). 
	\end{align}
	The result follows by from (\ref{eq:upperBoundOutTemp4}) by setting $\delta'= \frac{\delta}{\delta + {\rm diam}(S)}$.
\end{proof}

\subsubsection{Proof of Theorem \ref{th:determinist}}
\begin{proof}
				Let us first prove that $\lim_{k \to \infty}d_H(S,S_k) = 0$. We take care of both expressions in the definition of $d_H$ separately. Fix an arbitrary $k \in \NN$, from Assumption \ref{ass:dk} and Lemma \ref{lem:upperBoundOut}, for any $\x \in \RR^p$ such that ${\rm dist}(\x,S)> \delta_k$,
	\begin{align*}
		s(d_k) \Lambda_{\mu_S,d_k}(\x) &\leq 2^{3 - \frac{{\rm dist}(\x,S) d_k}{{\rm dist}(\x,S) + {\rm diam}(S)}} d_k^p \left( \frac{e}{p} \right)^p \exp\left( \frac{p^2}{d_k} \right)\\
		&< \alpha_k.
	\end{align*}
	From this we deduce that $\R^p\setminus S_k \supseteq \left\{ \x \in \RR^p:{\rm dist}(\x,S) > \delta_k \right\}$ and thus $S_k \subseteq \left\{ \x \in \RR^p:{\rm dist}(\x,S) \leq \delta_k \right\}$. Since $k$ was arbitrary, for any $k \in \NN$,
	\begin{align}
		\label{eq:deterministTemp1}
		\sup_{\x \in S_k}{\rm dist}(\x, S) \leq \delta_k.
	\end{align}
	Inequality (\ref{eq:deterministTemp1}) allows to take care of one term in the expression of $d_H$. Let us now consider the second term. We would like to show that 
	\begin{align}
		\label{eq:deterministTemp2}
		\sup_{\x \in S}{\rm dist}(\x, S_k) \to 0\quad\mbox{as $k\to \infty$.}
	\end{align}
	Note that the supremum is attained in (\ref{eq:deterministTemp2}). We will prove this by contradiction, for the rest of the proof, $M$ denotes a fixed positive number which value can change between expressions. Suppose that (\ref{eq:deterministTemp2}) is false. This means that for each $k \in \NN$ (up to a subsequence), we can find $\x_k \in S$ which satisfies
	\begin{align}
		\label{eq:deterministTemp3}
		{\rm dist}(\x_k, S_k) \geq M
	\end{align}
	Since $\x_k \in S$ and $S$ is compact, the sequence $(\x_k)_{k\in\NN}$ has an accumulation point $\bar{\x} \in S$, i.e., (up to a subsequence) $\x_k \to \bar{\x}$ as $k\to\infty$. Since ${\rm dist}(\cdot, S_k)$ is a Lipschitz function, combining with (\ref{eq:deterministTemp3}), for every $k \in \NN$ (up to a subsequence),
	\begin{align}
		\label{eq:deterministTemp4}
		{\rm dist}(\bar{\x}, S_k) \geq M.
	\end{align}
	We next show that (\ref{eq:deterministTemp4}) contradicts the assumption $S = {\rm cl}({\rm int}(S))$. From now on, we discard terms not in the subsequence and assume that (\ref{eq:deterministTemp4}) holds for all $k \in \NN$. Combining Lemma \ref{lem:lowerBoundIn} and Assumption \ref{ass:dk}, for every $k \in \NN$
	\begin{align}
		\label{eq:deterministTemp5}
		S_k	\supseteq \left\{ \x \in S: {\rm dist}(\x, \partial S) \geq \delta_k \right\}.
	\end{align}
	Since $S = {\rm cl}({\rm int}(S))$ and $\bar{\x} \in S$, consider a sequence $\{\y_l\}_{l \in \NN} \subset {\rm int}(S)$ such that $\y_l \to \bar{\x}$ as $l\to\infty$. Since $\y_l \in {\rm int}(S)$, we have ${\rm dist}(\y_l, \partial S) > 0$ for all $l$. Up to a rearrangment of the terms, we may assume that ${\rm dist}(\y_l, \partial S)$ is decreasing and ${\rm dist}(\y_0, \partial S) \geq \delta_0$. For all $l$, denote by $k_l$ the smallest integer such that ${\rm dist}(\y_l, \partial S) \geq \delta_{k_l}$. We must have $k_l \to \infty$ and we can discard terms so that $k_l$ is a valid subsequence. We have constructed a subsequence $k_l$ such that for every $l \in \NN$, $\y_l \in S_{k_l}$ and $\y_l \to \bar{\x}$. This is in contradiction with (\ref{eq:deterministTemp4}) and hence (\ref{eq:deterministTemp2}) must be true. Combining (\ref{eq:deterministTemp1}) and (\ref{eq:deterministTemp2}) we have that $\lim_{k \to \infty} d_H(S,S_k) = 0$. 

	Let us now prove that $\lim_{k \to \infty}d_H(\partial S,\partial S_k) = 0$, we begin with the term $\sup_{\x \in \partial S_k} {\rm dist}(\x, \partial S)$. Fix an arbitrary $k \in \NN$ and $\bar{\x} \in \partial S_k$. We will distinguish the cases $\bar{\x} \in S$ and $\bar{\x} \not\in S$. Assume first that $\bar{\x} \not \in S$. We deduce from (\ref{eq:deterministTemp1}), that
	\begin{align}
		\label{eq:deterministTemp6}
		{\rm dist}(\bar{\x}, \partial S) = {\rm dist}(\bar{\x}, S) \leq \delta_k.
	\end{align}
	Assume now that $\bar{\x} \in S$. If $\bar{\x} \in \partial S$, we have ${\rm dist}(\bar{\x}, \partial S) = 0$. Assume that $\bar{\x} \in {\rm int}(S)$. From (\ref{eq:deterministTemp5}), we have that $S \setminus S_k \subseteq \left\{ \x \in S: {\rm dist}(\x, \partial S) < \delta_k   \right\}$ and hence ${\rm cl}(S \setminus S_k) \subseteq \left\{ \x \in S: {\rm dist}(\x, \partial S) \leq \delta_k   \right\}$.  Since $\bar{\x} \in \partial S_k \cap {\rm int}(S)$, we have $\bar{\x} \in {\rm cl}(S \setminus S_k)$ and hence ${\rm dist}(\bar{\x}, \partial S) \leq \delta_k$. Combining the two cases $\bar{x} \in S$ and $\bar{x} \not\in S$, we have in any case that ${\rm dist}(\bar{\x}, \partial S) \leq \delta_k$ and hence
	\begin{align}
		\label{eq:deterministTemp7}
		\sup_{\x \in \partial S_k}{\rm dist}(\x, \partial S) \leq \delta_k.
	\end{align}
	Let us now prove that 
	\begin{align}
		\label{eq:deterministTemp8}
		\sup_{\x \in \partial S}{\rm dist}(\x, \partial S_k) \to 0\quad\mbox{as $k\to \infty$.}
	\end{align}
	First since $S$ is closed by asumption, the supremum is attained for each $k \in \NN$. Assume that (\ref{eq:deterministTemp8}) does not hold, this means there exists a constant $M >0$, such that we can find $\x_k \in \partial S$, $k \in \NN$ with ${\rm dist}(\x_k, \partial S_k) \geq M$. If $\x_k \not\in S_k$ infinitely often, then, we would have up to a subsequence $\x_k \in S$ and ${\rm dist}(\x_k, S_k) \geq M$. This is exactly (\ref{eq:deterministTemp3}) and we alredy proved that it cannot hold true. Hence, $\x_k \not\in S_k$ only finitely many times and we may assume by discarding finitely many terms that $\x_k \in S_k$ for all $k \in \NN$. Let $\bar{\x} \in \partial S$ be an accumulation point of $(\x_k)_{k \in \NN}$. Since $\bar{\x} \in \partial S$, there exists $\bar{\y} \not \in S$ such that $0 < {\rm dist}(\bar{\y},S) \leq \lim\inf_{k \to \infty} {\rm dist}(\x_k, \partial S_k) / 2$. Since $\x_k \in S_k$ for all $k$ sufficiently large, we have $\bar{\y} \in S_k$ for all $k$ sufficiently large but the fact that $0 < {\rm dist}(\bar{\y},S)$ contradicts (\ref{eq:deterministTemp1}). Hence (\ref{eq:deterministTemp8}) must hold true and the proof is complete.
\end{proof}

\begin{rem}[Refinements]
	The proof of Theorem \ref{th:determinist} is based on the following fact
	\begin{align*}
		\left\{ \x \in \RR^p :\: {\rm dist}(\x, \bar{S}) \geq \delta_k \right\} \subseteq S_k \subseteq \left\{ \x \in \RR^p:\:{\rm dist}(\x,S) \leq \delta_k \right\}.
	\end{align*}
	Depending on the regularity of the boundary $\partial S$ of $S$, it should be possible to get sharper bounds on the distance as a function of $\delta_k$. This should involve the dependency on $\delta$ of the function
	\begin{align*}
		\delta \to d_H\left( \left\{ \x \in \RR^p :\: {\rm dist}(\x, \bar{S}) \geq \delta \right\}, \partial S\right).
	\end{align*}
	For example, if the boundary $\partial S$ has bounded curvature, this function is equal to $\delta$ for sufficiently small $\delta$. Another example, if $S \subset \RR^2$ is the interior region of a non self intersecting continuous polygonal loop, then the function is of the order of $\frac{\delta}{\sin\left( \frac{\theta}{2} \right)}$, where $\theta$ is the smallest angle between two consecutive segments of the loop.
\end{rem}
\subsubsection{Proof of Theorem \ref{th:determinist2}}
\begin{proof}
	Lemma \ref{lem:lowerBoundIn} holds with $\mu$ in place of $\mu_S$ and $w_-$ in place of $\frac{1}{\lambda(S)}$. Indeed, we have 
	\begin{align*}
		\Lambda_{\mu,d} \geq w_- \lambda(\B_{\delta}(\x)) \Lambda_{\mu_{\B_\delta(\x)}},
	\end{align*}
	and the rest of the proof remains the same with different constants. Similarly, Lemma \ref{lem:upperBoundOut} holds with $\mu$ in place of $\mu_S$, indeed, the proof only uses the fact that $\mu_S$ is a probability measure supported on $S$ which is also true for $\mu$. The proof then is identical to that of Theorem \ref{th:determinist} by reflecting the corresponding change in the constants.
\end{proof}

\subsection{Proof of Theorem \ref{th:discretization}}

\subsubsection{A preliminary Lemma}
\begin{lem}
	\label{lem:prelimDiscrete}
	Let $\mu$ be a probability measure supported on a compact set $S$. Then for every $d \in \NN$, $d > 0$,  and every $\x \in \RR^p$,
	\begin{align*}
		\Lambda_{\mu,d}(\x) \leq \left(\frac{{\rm diam}({\rm conv}(S))}{{\rm dist}(\x,{\rm conv}(S)) +{\rm diam}({\rm conv}(S))}  \right)^2.
	\end{align*}
\end{lem}
\begin{proof}
	Set $\y = {\rm proj}(\x,{\rm conv}(S))$, that is $\|\y - \x\|\,=\, {\rm dist}(\x, {\rm conv}(S))$ and:
	\begin{equation}\label{eq:prelimDiscrete1}
		\y\,= \displaystyle{\arg\min}_{\z \in {\rm conv}(S)} \{\,\left\langle \z, \y - \x\right\rangle\}.\end{equation}
	Consider the affine function
	\begin{align}	
		\label{eq:prelimDiscrete2}
		\z\mapsto f_\x(\z) \,:=\, \frac{\left\langle \x - \z, \frac{\x - \y}{\|\x - \y\|}\right\rangle}{\|\x - \y\| + {\rm diam}({\rm conv}(S))}.
	\end{align}
	For any $\z \in S$, we have
	\begin{align}	
		\label{eq:prelimDiscrete3}
		f_\x(\z) \leq \frac{\|\x - \z\|}{\|\x-\y\| + {\rm diam}({\rm conv}(S))} \leq \frac{\|\x - \y\| + \|\y - \z\|}{\|\x-\y\| + {\rm diam}({\rm conv}(S))} \leq 1,
	\end{align}
	where we have used Cauchy-Schwartz and triangular inequalities. Furthermore, we have for any $\z \in S$,
	\begin{align}
		\label{eq:prelimDiscrete4}
		f_\x(\z) \geq \min_{\z\in {\rm conv}(S)} f_\x(\z) = \frac{\|\x - \y\|}{\|\x - \y\| + {\rm diam}({\rm conv}(S))},
	\end{align}
	where we have used equation (\ref{eq:prelimDiscrete1}). Consider the affine function $q_\x \colon \z \to 1 - f_\x(\z)$. We have
	\begin{align}
		\label{eq:prelimDiscrete5}
		&q_\x(\x)=1\\
		&0\leq q_\x(\z) \leq \frac{{\rm diam}({\rm conv}(S))}{\|\x - \y\| + {\rm diam}({\rm conv}(S))} \nonumber, \text{ for any } \z \in S,
	\end{align}
	where the inequalities are obtained by combining (\ref{eq:prelimDiscrete3}) and (\ref{eq:prelimDiscrete4}). The result follows from (\ref{eq:prelimDiscrete1}), (\ref{eq:prelimDiscrete5}) and Theorem \ref{th:christoffel}.
\end{proof}

\subsubsection{Proof of Theorem \ref{th:discretization}}.

\begin{proof}
	First let us consider measurability issues. Fix $n$ and $d$ such that $M_d(\mu)$ is invertible. Let $\X$ be a matrix in $\RR^{p\times n}$, we use the shorthand notation 
	\begin{align}
		\label{eq:discretization00}
		\Lambda_{\X,d}(\bz) = \min_{P \in \RR_d[\bx],\,P(\z)=1} \frac{1}{n}\sum_{i=1}^n P(\X_i)^2,
	\end{align}
	where for each $i$, $\X_i$ is the $i$-th column of the matrix $\X$. This corresponds to the empirical Christoffel function with input data given by the columns of $\X$. 
	Consider the function $F\colon \RR^{p\times n}\to [0,1]$ defined as follows:
	\begin{align}
		\label{eq:discretization0}
		F\colon \X \to \sup_{\z \in \RR^p} \left| \Lambda_{\mu,d}(\z) - \Lambda_{\X,d}(\z)\right|.
	\end{align}
It turns out that $F$ is a semi-algebraic function (its graph is a semi-algebraic set). Roughly speaking a set is semi-algebraic if it can be defined by finitely many polynomial inequalities. We refer the reader to \cite{coste2000introduction} for an introduction to semi-algebraic geometry, we mostly rely on content from Chapter 2. First, the function
	\begin{align*}
		(\X,\bz,P) \to \frac{1}{n}\sum_{i=1}^n P(\X_i)^2
	\end{align*}
	is semi-algebraic (by identifying the space of polynomials with the Euclidean space of their coefficients) and the set $\left\{ (P,\bz): P(\z)=1 \right\}$ is also semi-algebraic. Constrained partial minimization can be expressed by a first order formula, and, by Tarski-Seidenberg Theorem (see \textit{e.g.} \cite[Theorem 2.6]{coste2000introduction}), this operation preserves semi-algebraicity. Hence, the function $(\X,\bz) \to \Lambda_{\X,d}(\bz)$ is semi-algebraic. Furthermore, Theorem \ref{th:christoffel} ensures that $\Lambda_{\mu,d}(\bz) = 1/\kappa(\bz,\bz)$ for any $\bz$, where $\kappa(\bz,\bz)$ is a polynomial in $\bz$ and hence $\bz \to \Lambda_{\mu,d}(\bz)$ is semi-algebraic. Finally absolute value is semi-algebraic and using a partial minimization argument again, we have that $F$ is a semi-algebraic function. 
	 
	As a semi-algebraic function, $F$ is Borel measurable. Indeed, using the {\it good sets principle} (\cite{Ash} \S 1.5.1, p. 35) it is sufficient to prove that, for an arbitrary interval\footnote{Recall that the Borel $\sigma$-field $\mathcal{B}([0,1])$ is generated by the intervals $(a,b]$ of $[0,1]$; see \cite{Ash} \S 1.4.6, p. 27.} $(a,b]\subset [0,1]$, $F^{-1}((a,b])\in\mathcal{B}(\R^{p\times n})$. Any such set is the pre-image of a semi-algebraic set by a semi-algebraic map. As proved in \cite[Corollary 2.9]{coste2000introduction}, any such set must be semi-algebraic and hence measurable. Thus, with the notations of Theorem \ref{th:discretization}, $\|\Lambda_{\mu_n,d} - \Lambda_{\mu,d}\|_{\infty}$ is indeed a random variable for each fixed $n,d$ such that $M_d(\mu)$ is invertible.
	
	We now turn to the proof of the main result of the Theorem. For simplicity we adopt the following notation for the rest of the proof. For any continuous function $f: \RR^p \to \RR$, and any subset $V \subseteq \RR^p$, 
	\begin{align}
		\label{eq:discretization1}
		\|f\|_{V} &:= \sup_{\x \in V}\vert f(\x)\vert,\qquad \mbox{[ so that $\Vert f\Vert_{\R^p}=\Vert f\Vert_\infty$]},
	\end{align}
	which could be infinite. We prove that for any $\epsilon >0$,
	\begin{align}	
		\label{eq:discretization2}
		P\left( {\lim\sup}_n\left\{ \|\Lambda_{\mu_n,d} - \Lambda_{\mu,d}\|_{{\RR^p}} \geq \epsilon \right\} \right) = 0,
	\end{align}
	where the probability is taken with respect to the random choice of the sequence of independent samples from $\mu$ and the limit supremum is the set theoretic limit of the underlying events. 
	
	Fix $\epsilon >0$. Denote by $S$ the compact support of $\mu$. Note that $S$ contains also the support of $\mu_n$ with probability one. From Lemma \ref{lem:prelimDiscrete}, we have an upper bound on both $\Lambda_{\mu_n,d}$ and $\Lambda_{\mu,d}$ of order $O\left({\rm dist}(\x, {\rm conv}(S))^{-2} \right)$ which holds with probability one. Hence, it is possible to find a compact set $V_{\epsilon}$ containing $S$ (with complement $V_\epsilon^c=\R^n\setminus V_\epsilon$) 
		such that, almost surely, 
	\begin{align}
		\label{eq:discretization3}
		\max\left\{ \|\Lambda_{\mu_n,d}\|_{V_\epsilon^c}, \|\Lambda_{\mu,d}\|_{V_\epsilon^c} \right\} \leq \frac{\epsilon}{2}.
	\end{align}
	Next, we have the following equivalence 
	\begin{align}
		\label{eq:discretization4}
		\|\Lambda_{\mu_n,d} - \Lambda_{\mu,d}\|_{\RR^p} \geq \epsilon \quad \Leftrightarrow\quad\left\lbrace 
		\begin{array}{l}
			\|\Lambda_{\mu_n,d} - \Lambda_{\mu,d}\|_{V_\epsilon} \geq \epsilon \text{ or}\\
			\|\Lambda_{\mu_n,d} - \Lambda_{\mu,d}\|_{V_\epsilon^c}\geq \epsilon
		\end{array}
		\right.
	\end{align}
	On the other hand, since both functions are non negative, from equation (\ref{eq:discretization3}), almost surely,
	\begin{align}
		\label{eq:discretization5}
		\|\Lambda_{\mu_n,d} - \Lambda_{\mu,d}\|_{V_\epsilon^c} \leq \max\left\{ \|\Lambda_{\mu_n,d}\|_{V_\epsilon^c}, \|\Lambda_{\mu,d}\|_{V_\epsilon^c} \right\}\leq \frac{\epsilon}{2}.
	\end{align}
	Hence the second event in the right hand side of (\ref{eq:discretization4}) occurs with probability zero. As a consequence, except for a set of events of measure zero, we have
\[		\|\Lambda_{\mu_n,d} - \Lambda_{\mu,d}\|_{\RR^p} \geq \epsilon \Leftrightarrow \|\Lambda_{\mu_n,d} - \Lambda_{\mu,d}\|_{V_\epsilon} \geq \epsilon,\]
	which in turn implies 
	\begin{align}
		\label{eq:discretization7}
		P\left({\lim\sup}_n \left\{\|\Lambda_{\mu_n,d} - \Lambda_{\mu,d}\|_{\RR^p}  \right\} \geq \epsilon\right) = P\left({\lim\sup}_n \left\{\|\Lambda_{\mu_n,d} - \Lambda_{\mu,d}\|_{V_{\epsilon}}  \right\} \geq \epsilon\right).
	\end{align}
	By assumption the moment matrix $\M_d(\mu)$ is invertible and by the strong law of large numbers, almost surely, $\M_d(\mu_n)$ must be invertible for sufficiently large $n$. Assume that $\M_d(\mu_n)$ is invertible, we have
	\begin{align}
		\label{eq:discretization8}
		\|\Lambda_{\mu_n,d} - \Lambda_{\mu,d}\|_{V_{\epsilon}} &= \sup_{\x \in V_{\epsilon}} \left\{\left|\frac{1}{\v_d(\x)^T\M_d(\mu)^{-1}\v_d(\x)} - \frac{1}{\v_d(\x)^T\M_d(\mu_n)^{-1}\v_d(\x)}\right|\right\}\\
		&= \sup_{\x \in V_{\epsilon}} \left\{\left|\frac{\v_d(\x)^T(\M_d(\mu_n)^{-1}- \M_d(\mu)^{-1})\v_d(\x)}{\v_d(\x)^T\M_d(\mu)^{-1}\v_d(\x)\v_d(\x)^T\M_d(\mu_n)^{-1}\v_d(\x)  }\right|\right\}\nonumber.
	\end{align}
	Using the strong law of large numbers again, continuity of eigenvalues and the fact that for large enough $n$, $\M_d(\mu_n)$ is invertible with probability one, the continuous mapping theorem ensures that almost surely, for $n$ sufficiently large, the smallest eigenvalue of $\M_d(\mu_n)^{-1}$ is close to that of $\M_d(\mu)^{-1}$ and hence bounded away from zero. Since the first coordinate of $\v_d(\x)$ is $1$, the denominator in (\ref{eq:discretization8}) is bounded away from zero almost surely for sufficiently large $n$. In addition, since $V_\epsilon$ is compact, $\v_d(\x)$ is bounded on $V_\epsilon$ and there exists a constant $K$ such that, almost surely, for sufficiently large $n$,
	\begin{align}
		\label{eq:discretization9}
		\|\Lambda_{\mu_n,d} - \Lambda_{\mu,d}\|_{V_{\epsilon}} \leq K\|\M_d(\mu)^{-1} - \M_d(\mu_n)^{-1}\|,
	\end{align}
	where the matrix-norm in the right hand side is the operator norm induced by the Euclidean norm. Combining (\ref{eq:discretization7}) and (\ref{eq:discretization9}), we obtain
	\begin{align}	
		\label{eq:discretization10}
		&P\left({\lim\sup}_n \left\{\|\Lambda_{\mu_n,d} - \Lambda_{\mu,d}\|_{\RR^p}  \right\} \geq \epsilon\right) \nonumber\\
		\leq\;&P\left({\lim\sup}_n \left\{K\Vert\M_d(\mu)^{-1} - \M_d(\mu_n)^{-1}\Vert\geq \epsilon\right\}\right).
	\end{align}
	The strong law of large numbers and the continuity of the matrix inverse $\M_d(\cdot)^{-1}$ at $\mu$ ensure that the right hand side of (\ref{eq:discretization10}) is $0$. This concludes the proof.
\end{proof}

\end{document}